\documentclass[11pt]{article}
\pdfoutput=1

\usepackage[margin=1in]{geometry}
\usepackage[numbers]{natbib}
\def\ourtitle{As you like it: Localization via paired comparisons}

\def\abstracttext{Suppose that we wish to estimate a vector $\x$ from a set of binary paired
comparisons of the form ``$\x$ is closer to $\p$ than to $\q$'' for various
choices of vectors $\p$ and $\q$. The problem of estimating $\x$ from this
type of observation arises in a variety of contexts, including nonmetric
multidimensional scaling, ``unfolding,'' and ranking problems, often because
it provides a powerful and flexible model of preference.  We describe
theoretical bounds for how well we can expect to estimate $\x$ under a
randomized model for $\p$ and $\q$. We also present results for the case where
the comparisons are noisy and subject to some degree of error.  Additionally,
we show that under a randomized model for $\p$ and $\q$, a suitable number of
binary paired comparisons yield a stable embedding of the space of target
vectors.  Finally, we also show that we can achieve significant gains
by adaptively changing the distribution used for choosing $\p$ and $\q$.
}

\usepackage{graphicx}

\usepackage{wrapfig}

\usepackage{xcolor}

\definecolor{darkred}{RGB}{100,0,0}
\definecolor{darkgreen}{RGB}{0,100,0}
\definecolor{darkblue}{RGB}{0,0,150}

\usepackage[utf8]{inputenc}
\usepackage[T1]{fontenc}

\usepackage{amssymb}
\usepackage{amsmath, amsthm, mathtools}
\usepackage{amsbsy}

\newtheorem{theorem}{Theorem}
\newtheorem{lemma}{Lemma}
\newtheorem{proposition}{Proposition}

\input{glyphtounicode.tex}
\input{glyphtounicode-cmr.tex}
\usepackage{cmap}

\usepackage{lmodern}
\usepackage[protrusion=true,expansion=true]{microtype}
\usepackage{booktabs}

\newcommand{\R}{\mathbb{R}}

\renewcommand{\vec}[1]{\mathbf{#1}}  \newcommand{\mat}[1]{\mathbf{#1}}  
\newcommand{\norm}[1]{\left \|#1\right \|}

\newcommand{\twonorm}[1]{\norm{#1}}  

\newcommand{\der}{\,\mathrm{d}}

\newcommand{\pr}{\mathbb{P}}
\newcommand{\pd}[2]{\frac{\partial #1}{\partial #2}}
\DeclareMathOperator{\sign}{sign}

\DeclareMathOperator*{\argmin}{arg\,min}
\DeclarePairedDelimiter{\ceil}{\lceil}{\rceil}
\DeclarePairedDelimiter{\flr}{\lfloor}{\rfloor}
\DeclareMathOperator{\E}{\mathbb{E}}
\DeclareMathOperator{\Vol}{Vol}
\DeclareMathOperator{\Diam}{Diam}

\newcommand{\bal}{\!\begin{aligned}}
\newcommand{\eal}{\end{aligned}}

\def\N{\mathcal{N}}
\def\distrib#1{\mathbb{\1}}

\def\bzero{\boldsymbol{0}}

\def\half{\frac{1}{2}}
 
\usepackage[pdfauthor={Massimino \& Davenport},
            pdftitle={\ourtitle},
            colorlinks=true,
            linkcolor=darkred,
            citecolor=darkgreen,
            urlcolor=darkblue,
            pdfa]{hyperref}

\def\cA{\mathcal{A}}\def\cN{\mathcal{N}}
\def\a{\vec{a}}\def\b{\vec{b}}
\def\p{\vec{p}}\def\q{\vec{q}}
\def\x{\vec{x}}\def\y{\vec{y}}
\def\w{\vec{w}}\def\z{\vec{z}}
\def\u{\vec{u}}\def\v{\vec{v}}

\def\S{\mathbb{S}}  \def\B{\mathbb{B}}  \def\vx{\vec{x}}

\newcommand{\acks}[1]{\section*{Acknowledgments}
    #1
}

\begin{document}

\title{\ourtitle}

\def\massemail{\href{mailto:massimino@gatech.edu}{massimino@gatech.edu}}
\def\mdavemail{\href{mailto:mdav@gatech.edu}{mdav@gatech.edu}}

\author{Andrew~K.~Massimino and Mark~A.~Davenport}

\maketitle

\begin{abstract}\abstracttext
\end{abstract}

\section{Introduction}\label{sec:intro}

In this paper we consider the problem of determining the location of a point
in Euclidean space based on distance comparisons to a set of known points,
where our observations are nonmetric. In particular, let $\x\in\R^n$ be the true position of
the point that we are trying to estimate, and let
$(\p_1,\q_1), \ldots, (\p_m, \q_m)$ be pairs of ``landmark'' points in $\R^n$ which we assume to be known {\em a priori}.
  Rather than directly observing the raw distances from $\x$, i.e.,
$\|\x-\p_i\|$ and $\|\x-\q_i\|$, we instead obtain only paired comparisons of
the form $\|\x-\p_i\|<\|\x-\q_i\|$.  Our goal is to estimate $\x$ from a
set of such inequalities.  Nonmetric observations of this type arise in
numerous applications and
have seen considerable interest in recent literature
e.g., \cite{ailon2011activeranking, Davenport13, eriksson2013learning, shah2016estimation}.
These methods are often applied in situations where we have a collection of
items and hypothesize that it is possible to embed the items in $\R^n$ in such
a way that the Euclidean distance between points corresponds to their
``dissimilarity,'' with small distances corresponding to similar items.

Here, we focus on the sub-problem of adding a new point to a
known (or previously learned) configuration of landmark points.

As a motivating example, we consider the problem of estimating a
user's preferences from limited response data.
This is useful, for instance, in recommender systems,
information retrieval, targeted advertising, and psychological studies.
A common and intuitively appealing way to model preferences
is via the \emph{ideal point model}, which supposes preference for a particular
item varies inversely with Euclidean distance in a feature
space~\cite{coombs1950psychological}.
We assume that the items to be rated are represented by
points $\p_i$ and $\q_i$ in an $n$-dimensional Euclidean space.
A user's preference is
modeled as an additional point $\x$
in this space (called the individual's
``ideal point'').  This represents a hypothetical ``perfect'' item
satisfying all of the user's criteria for evaluating items.

Using response data consisting of paired comparisons between items
(e.g., ``user $\x$ prefers item $\p_i$ to item $\q_i$'')
is a natural approach when
dealing with human subjects since it avoids requiring people to assign precise
numerical scores to different items, which is generally a quite difficult
task, especially when preferences may depend on multiple
factors~\cite{miller1956magical}.  In contrast, human subjects often find
pairwise judgements much easier to make~\cite{david1963method}.
Data consisting of paired comparisons is often generated implicitly in
contexts where the user has the option to act on two (or more) alternatives;
for instance they may choose to watch a particular movie, or click a
particular advertisement, out of those displayed to
them~\cite{radlinski2007active}. In such contexts, the ``true distances'' in
the ideal point model's preference space are generally inaccessible directly,
but it is nevertheless still possible to obtain an estimate of
a user's ideal point.

\begin{figure}[tbp]\centering
    \includegraphics[width=2in]{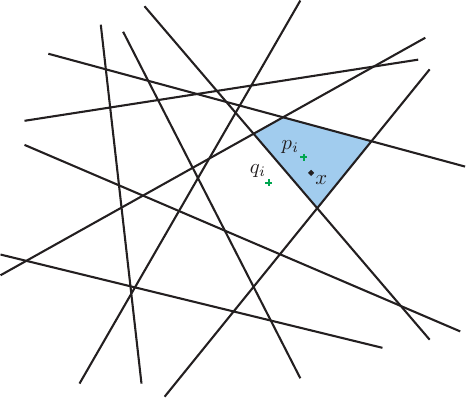}
    \caption{An illustration of the localization problem from paired
    comparisons. The information that $\x$ is closer to $\p_i$ than $\q_i$
tells us which side of a hyperplane $\x$ lies.  Through many such comparisons
we can hope to localize $\x$ to a high degree of accuracy.}
    \label{fig:hyper}
\end{figure}

\subsection{Main results}

The fundamental question which interests us in this paper is how many
comparisons we need (and how should we choose them) to estimate $\x$ to a
desired degree of accuracy. Thus, we consider the case where we are given an
existing embedding of the items (as in a mature recommender system) and focus
on the on-line problem of locating a single new user from their feedback
(consisting of binary data generated from paired comparisons). The item
embedding could be generated using various methods, such as
multidimensional scaling applied to a set of item features, or even using the
results of previous paired comparisons via an approach like that
in~\cite{AgarwWCLKB_Generalized}. Given such an embedding of $\ell$ items,
there are a total of  $\binom{\ell}{2} = \Theta(\ell^2)$ possible paired
comparisons. Clearly, in a system with thousands (or more) items, it will be
prohibitive to acquire this many comparisons as a typical user will likely
only provide comparisons for a handful of items.
Fortunately, in general we can expect that many, if not most, of the possible
comparisons are actually redundant.  For example, of the comparisons
illustrated in Fig.~\ref{fig:hyper}, all but four are redundant and---at least in the absence of noise---add no
additional information.

Any precise answer to this question would depend on the underlying
geometry of the item embedding.
Each comparison essentially divides $\R^n$ in two, indicating
on which side of a hyperplane $\x$ lies,
and some arrangements of hyperplanes will yield better
tessellations of the preference space than others.
Thus, to gain some intuition on this problem without reference to the geometry
of a particular embedding, we will instead consider a probabilistic model
where the items are generated at random from a particular distribution.  In
this case we show that under certain natural assumptions on the distribution,
it is possible to estimate the location of any $\x$ to within an error of
$\epsilon$ using a number of comparisons which, up to log factors, is
proportional to $n/\epsilon$. This is
essentially optimal, so that no set of comparisons can provide a uniform
guarantee with significantly fewer comparisons.  We then describe several stability and robustness guarantees for various settings in which the comparisons are subject to noise or errors.
Finally, we then describe a simple
extension to an {\em adaptive} scheme where we
adaptively select the comparisons (manifested here in adaptively altering the
mean and variance of the distribution generating the items) to
substantially reduce the required number of comparisons.

\subsection{Related work}

It is important to note that the ideal point model, while similar, is distinct
from the low-rank model used in \emph{matrix completion}~\cite{rennie2005fast,candes2009exact}.  Although
both models suppose user choices are guided by a number of attributes, the
ideal point model leads to preferences that are {\em non-monotonic} functions
of those attributes. The ideal point model suggests that each
feature has an ideal level; too much of
a feature can be just as undesirable as too little.  It is not
possible to obtain this kind of performance with a traditional low-rank model,
though if points are limited to the sphere, then the ideal point model
can duplicate the performance of a low-rank factorization.
There is also empirical evidence
that the ideal point model captures behavior more accurately than
factorization based approaches do~\cite{dubois1975ideal,maydeu2009modeling}.

There is a large body of work that
studies the problem of learning to rank items from various sources of data,
including paired comparisons of the sort we consider in this paper.  See, for
example,~\cite{activeranking,jamieson2011low,wauthier2013efficient} and
references therein.  We first note that in most work on rankings, the central
focus is on learning a correct rank-ordered list for a particular user,
without providing any guarantees on recovering a correct parameterization for
the user's preferences as we do here.  While these two problems are
related, there are natural settings where it might be desirable to guarantee
an accurate recovery of the underlying parameterization ($\x$ in our
model). For example, one could exploit these guarantees in the context of an
iterative algorithm for nonmetric multidimensional scaling which aims to
refine the underlying embedding by updating each user and item one at
a time \citep[e.g., see][]{OShaughnessy}, in
which case an understanding of the error in the estimate of $\x$ is crucial.
Moreover, we believe that our approach provides an interesting
alternative perspective as it yields natural
robustness guarantees and suggests simple adaptive schemes.

Also closely related is the work
in~\cite{LuN_Individualized,ParkNZSD_Preference,OhTX_Collaboratively} which
consider paired comparisons and more general ordinal measurements in the
similar (but as discussed above, subtly different) context of low-rank
factorizations.
Perhaps most closely related to our work is that
of~\cite{activeranking}, which examines the problem of
learning a rank ordering using the same ideal point model considered in this
paper. The message in this work is broadly consistent with ours, in that the
number of comparisons required should scale with the dimension of the
preference space (not the total number of items) and can be significantly
improved via a clever adaptive scheme.
However, this work does not bound the estimation error in terms of the
Euclidean distance, which is our central concern.
\cite{jamieson2011low} also incorporates adaptivity, but
seeks to embed a set of points in Euclidean space (as opposed to
estimating a
single user's ideal point) and relies on paired comparisons involving three arbitrarily selected points (rather than a user's ideal point and two items).
Our dyadic adaptive strategy is also similar
    in spirit to a higher-dimensional form of binary search, as in
    \cite{nowak2009noisy}, however here we consider estimating a
    continuous ideal point $\x$ rather than choosing from a finite set of
    possible hypotheses.

Constructing an embedding of items given comparison measurements
    is sometimes referred to as \emph{ordinal embedding}
    and is studied in
    \cite{kleindessner2014uniqueness}, \cite{jain2016finite}, and
    \cite{arias2017some}.
    As mentioned, in this work we assume the presence of an
    item embedding created by e.g., those methods.
Our work could then be used to create a corresponding embedding of
users based on response data to perform e.g., personalization and customer
segmentation.

Finally, while seemingly unrelated, we note that our work builds on the
growing body of literature of 1-bit compressive sensing.  In particular, our
results are largely inspired by those
in~\cite{onebitnormest-prepr,expdecay-prepr}, and borrow techniques
from~\cite{onebitstable} in the proofs of some of our main results.
Our embedding result of Section~\ref{sec:concentrate}
    is most directly related to the work of~\cite{plan2014dimension},
    which studies a similar problem to~\eqref{eq:measmodel0} but under a
    different, non-pairwise model.
    Because of this distinction, their results are not directly
    applicable.  However, due to our particular probabilistic model,
    our results are also in some sense stronger.
    We will expand further with a direct comparison to this work during
    the presentation of our result.
    The 1-bit sensing problem is also considered in \cite{plan2013robust}
    and \cite{plan2013one} but these works handle only queries involving
    homogeneous hyperplanes, without offset, which is not applicable to
    our pairwise setting.
    Ideas similar to 1-bit sensing appear in field of 
    locality sensitive hashing, e.g, \cite{konoshima2012hyperplane} which
    does treat inhomogeneous hyperplanes but does not provide theory.

Note that in this work we extend preliminary results first presented in~\cite{MassiD_Binary,MassiD_Geometry}.

 
\section{A randomized observation model}\label{sec:obsmodel}

For the moment we will consider the ``noise-free'' setting where each comparison between $\x$ and $\q_i$ versus $\p_i$ results in assigning the point which is truly closest to $\x$ with probability 1.
In this case we can represent the observed comparisons mathematically by letting $\cA_i(\x)$ denote the
$i^{\text{th}}$ observation, which consists of comparisons between $\p_i$ and
$\q_i$, and setting
\begin{equation} \label{eq:measmodel00}
 \cA_i(\x)  := \sign \left( \twonorm{ \x- \q_i}^2 -
    \twonorm{ \x - \p_i}^2 \right)
    = \begin{cases} +1 & \text{if $\x$ is
closer to $\p_i$} \\ -1 & \text{if $\x$ is closer to $\q_i$}. \end{cases}
\end{equation}
We will also use $\cA(\x) := [\cA_1(\x),\cdots,\cA_m(\x)]^T$
to denote the vector of all observations resulting from $m$ comparisons.
Note that since
\[
\twonorm{ \x - \q_i }^2 - \twonorm{ \x - \p_i}^2 = 2(\p_i - \q_i)^T \x  + \twonorm{\q_i}^2
    - \twonorm{\p_i}^2,
\]
if we set $\check\a_i = (\p_i - \q_i)$ and
$\check\tau_i = \frac12 (\twonorm{\p_i}^2 - \twonorm{\q_i}^2)$, then we can
re-write our observation model as
\begin{equation} \label{eq:measmodel0}
    \cA_i(\x) = \sign \left( 2\check\a_i^T \x - 2\check\tau_i
\right) = \sign \left( \check\a_i^T \x - \check\tau_i \right).
\end{equation}
This is reminiscent of
the standard setup in one-bit compressive sensing (with
dithers)~\cite{onebitnormest-prepr,expdecay-prepr}
with the
important differences that: {\em (i)} we have not made any kind of
sparsity or other structural assumption on $\x$ and, {\em (ii)} the
``dithers'' $\check\tau_i$, at least in this formulation, are dependent on the
$\check\a_i$, which results in difficulty applying standard results from
this theory to the present setting.

However, many of the techniques from this literature will nevertheless be
helpful in analyzing this problem. To see this, we consider a randomized
observation model where the
pairs $(\p_i, \q_i)$ are chosen independently with i.i.d.\ entries drawn
according to a normal distribution, i.e.,  $\p_i, \q_i \sim \cN(\bzero,
    \sigma^2 \mat I)$.  In this case, we have that the entries of our sensing
    vectors are i.i.d.\ with $\check a_{i}(j) \sim \cN(0,2\sigma^2)$.
    Moreover, if we
    define $\b_i = \p_i  + \q_i$, then we also have that
    $\b_i \sim \cN(0,
        2\sigma^2 \mat I)$, and
\[\bal \frac12 \check\a_i^T \b_i &= \frac12 \sum_{j}
    (\p_i(j) - \q_i(j))(\p_i(j) + \q_i(j))
    \\ &= \frac12 \sum_{j}  \p_i(j)^2 - \q_i(j)^2
    = \frac12 (\twonorm{\p_i}^2 - \twonorm{\q_i}^2) = \check\tau_i.
\eal\]
Note
that while $\check\tau_i = \frac12 \check\a_i^T \b_i$
is clearly dependent on $\check\a_i$, we
do have that $\check\a_i$ and $\b_i$ are independent.

To simplify, we re-normalize by dividing by
$\twonorm{\check\a_i}$, i.e., setting $\a_i
    := \check\a_i/\twonorm{\check\a_i}$ and
$\tau_i := \check\tau_i/\twonorm{\check\a_i}$, in which case we can write
\begin{equation} \label{eq:measmodel}
\cA_i(\x) = \sign ( \a_i^T \x
    - \tau_i ).
\end{equation}
It is easy to see that $\a_i$ is distributed
uniformly on the sphere $\S^{n-1} = \{\a\in\R^n : \twonorm{\a}=1\}$.  Note that throughout our analysis we will exploit the fact that $\a_i$ is uniform on $\S^{n-1}$ and will let $\nu$ denote the uniform measure on the sphere. Note also that
\[ \tau_i = \frac12 \a_i^T \b_i.  \] Since $\check\a_i$ and $\b_i$ are
independent, $\a_i$ and $\b_i$ are also independent.  Moreover, for any
unit-vector $\a_i$, if  $\b_i \sim \cN(0, 2\sigma^2 \mat I)$ then
$\a_i^T \b_i \sim \cN(0, 2\sigma^2)$.  Thus, we must have $\tau_i
\sim \cN(0, \sigma^2/2)$, independent of $\a_i$, which is the key insight that enables the analysis below.

 
\section{Guarantees in the noise-free setting}\label{sec:noiseless}

We now state  a result concerning localization under the noise-free
random model from Section~\ref{sec:obsmodel}.
Extensions to noisy comparisons and efficient estimation
in practice are discussed in Sections \ref{sec:noise} and
\ref{sec:reconstr}, respectively.
Let $\B^n_R$ denote the $n$-dimensional, radius $R$ Euclidean ball.

\begin{theorem}
\label{thm:rndmeasmain}
Let $\epsilon, \eta > 0$ be given. Let $\cA_i(\cdot)$ be defined as in~\eqref{eq:measmodel00}, and suppose that $m$ pairs $\{(\p_i, \q_i)\}_{i=1}^m$
are generated by drawing each $\p_i$ and $\q_i$
independently from $\N(0,\sigma^2 I)$ where $\sigma^2 = 2R^2/n$.
There exists a constant $C$ such that if
\begin{equation}\label{eq:sufficientm}
    m \ge C \frac{R}{\epsilon}
    \left( n\log\frac{R \sqrt{n}}{\epsilon}
        + \log\frac{1}{\eta} \right),
\end{equation}
then with probability at least $1-\eta$, for all $\x,\,\y\in
\B_R^n$ such that $\cA(\x) = \cA(\y)$,
\[
    \norm{\x-\y} \le \epsilon.
\]
\end{theorem}

The result follows from applying Lemma~\ref{lem:singlemeas} below
to pairs of points in a covering set of $\B_R^n$.
The key message of this theorem is that
if one chooses the variance $\sigma^2$ of the distribution generating the
items appropriately, then it is possible to estimate $\x$ to within $\epsilon$
using a number of comparisons that is nearly linear in $n/\epsilon$.
As we will show in Theorem~\ref{thm:lowerbound},
    this result is optimal, ignoring log factors, in terms of the 
    scaling of $m$ with $n$ and $\epsilon$.  This also
    makes intuitive sense; if all the hyperplanes were all axis-aligned,
    one would require the number of hyperplanes be at least
    proportional to the number of dimensions, $n$, otherwise some
    direction would be unconstrained.
    Theorem~\ref{thm:rndmeasmain} is also sensible
    in terms of the ratio of the initial uncertainty $R$ to the target
    uncertainty $\epsilon$ since $R/\epsilon$ hyperplanes would be
    required to uniformly localize a point along a single dimension.

A natural question is what would happen with a different choice of $\sigma^2$.
In fact, this assumption is critical---if $\sigma^2$ is substantially smaller the bound quickly becomes vacuous, and as $\sigma^2$ grows much past $R^2/n$ the bound begins to become steadily worse.\footnote{We note that it is possible to try to optimize $\sigma^2$ by setting $\sigma^2 = c R^2/n$ for some constant $c$ and then selecting $c$ so as to minimize the constant $C$ in~\eqref{eq:sufficientm}. We believe this would yield limited insight since, in order to obtain a result which is valid uniformly for all possible $n$, we use certain bounds which for general $n$ can be somewhat loose and would skew the resulting $c$. We instead simply select $c = 2$ for simplicity in our analysis (as it results in $\tau_i \sim \mathcal{N}(0,R^2/n)$) and because it aligns well with simulations.}  As we will see in Section~\ref{sec:experiments}, this is in fact observed in practice.  It should also be somewhat intuitive: if $\sigma^2$ is too small, then nearly all the hyperplanes induced by the comparisons will pass very close to the origin, so that accurate estimation of even $\| \x \|$ becomes impossible.  On the other hand, if $\sigma^2$ is too large, then an increasing number of these hyperplanes will not even intersect the ball of radius $R$ in which $\x$ is presumed to lie, thus yielding no new information.

\begin{lemma}
\label{lem:singlemeas}
Let $\w,\z\in \B^n_R$ be distinct and fixed, and let $\delta > 0$ be given.
Define
\[
    B_\delta(\w) := \{ \u \in\B_R^n : \norm{\u-\w}\le\delta \}.
\]
Let $\cA_i$ be defined as in Theorem~\ref{thm:rndmeasmain}. Denote by $P_{\text{{\em sep}}}$ the probability that $B_\delta(\w)$ and $B_\delta(\z)$ are separated by hyperplane $i$, i.e.,
\[ P_{\text{{\em sep}}} := \pr\left[
    \forall \u\in B_\delta(\w),\forall \v\in
    B_\delta(\z) :
    \cA_i(\u) \ne \cA_i(\v)
    \right].
\]
For any $\epsilon_0 \le \norm{\w-\z}$ we have
\[
    P_{\text{{\em sep}}}
    \ge \frac{\epsilon_0-\delta\sqrt{2n}}{22 \sqrt{\pi}e^{5/2} R}.
\]
\end{lemma}
\begin{proof}
Let $\epsilon = \norm{\w-\z}$. Here, we denote the normal vector and
threshold of hyperplane $i$ by $\a$ and $\tau$ respectively. It is easy to
show that $P_{\text{sep}}$ can be expressed as
\begin{align}
    P_{\text{sep}} &= \pr\left[ \a^T\z + \delta \le \tau \le \a^T\w-\delta
        ~~\text{or}~~ \a^T\w + \delta \le \tau \le \a^T\z-\delta \right] \notag \\
        & = 2 \pr\left[ \a^T\z + \delta \le \tau \le \a^T\w-\delta \right], \label{eq:prob1}
\end{align}
where the second equality follows from the symmetry of the distributions of $\a$ and $\tau$.

Define $C_\alpha := \{\a \in\S^{n-1} : \a^T(\w-\z) \ge \alpha \}$. Note that the probability in~\eqref{eq:prob1} is zero unless $\a \in C_{2\delta}$.  Thus,
recalling that $\tau_i
\sim \cN(0, \sigma^2/2)$ we have
\begin{align}
    P_{\text{sep}} & = 2 \int_{C_{2\delta}}
    \left|\Phi\left(\frac{\a^T\w-\delta}{\sigma/\sqrt2}\right)
    - \Phi\left(\frac{\a^T\z+\delta}{\sigma/\sqrt2}\right) \right|
    \nu(\mathrm{d}\a) \notag \\
    & \ge 2 \int_{C'}
    \left|\Phi\left(\frac{\a^T\w-\delta}{\sigma/\sqrt2}\right)
    - \Phi\left(\frac{\a^T\z+\delta}{\sigma/\sqrt2}\right) \right|
    \nu(\mathrm{d}\a) \label{eq:prob2}
\end{align}
for any $C' \subseteq C_{2\delta}$.
To obtain a lower bound on~\eqref{eq:prob2}, we will consider a carefully chosen subset
$C' \subseteq C_{2\delta}$ and then simply multiply the area of $C'$ by the minimum value $\gamma$ of the
integrand over that set, yielding a bound of the form
\[
    P_{\text{sep}} \ge 2 \gamma \nu(C').
\]
We construct the set $C'$ as follows. Let $W := \{\a : \a^T\w \le\xi/\sqrt{n}\norm{\w}\}$, $Z := \{\a: \a^T\z \ge -\xi/\sqrt{n}\norm{\z}\}$, and set $C' := C_\alpha \cap W \cap Z$ for some $\alpha \ge 2 \delta$. Note that for any $\a\in C'$, since $\a^T(\w-\z) \ge \alpha \ge 2\delta$,
we have $-R\xi/\sqrt{n} \le \a^T\z+\delta \le \a^T\w-\delta \le R\xi/\sqrt{n}$. Thus, by Lemma~\ref{lem:cdflower},
\[
    \gamma = \inf_{\a\in C'}
    \left|\Phi\left(\frac{\a^T\w-\delta}{\sigma/\sqrt{2}}\right)
    - \Phi\left(\frac{\a^T\z+\delta}{\sigma/\sqrt{2}}\right) \right|
    \ge \frac{\sqrt{2}}{\sigma} (\alpha-2\delta)
    \phi\biggl(\frac{\sqrt{2} R\xi}{\sigma \sqrt{n}}\biggr).
\]
Recall by assumption we have that
$\sigma = \sqrt{2}R/\sqrt{n}$, thus we obtain by
setting $\xi = \sqrt{5}$,
\begin{equation}
    \gamma \ge \frac{\sqrt{n}}{R} (\alpha - 2\delta) \phi(\xi)
    = \frac{ \sqrt{n}(\alpha - 2\delta)}{\sqrt{2\pi}e^{5/2}R}.
    \label{eq:P1_lb1_X}
\end{equation}

Next note that $C' = C_\alpha \cap W \cap Z = C_\alpha\setminus W^c
\setminus Z^c$ is a difference of a set of hyperspherical caps.
To obtain a lower bound on $\nu(C')$ we use the upper and lower bounds on the measure of hyperspherical caps given in
Lemma~2.1 of~\cite{brieden2001deterministic}.

\paragraph{Case $n\ge 6$}
Provided that $\alpha/\epsilon < \sqrt{2/n}$ we can bound
$\nu(C')$ as
\[\begin{aligned}
    \nu(C') &\ge \nu(C_{\alpha}) - \nu(W^c) - \nu(Z^c)
    \ge \frac{1}{12} - 2\frac{1}{2\xi}(1-\xi^2/n)^{(n-1)/2}
    \ge \frac{1}{12} - \frac{1}{\sqrt{5}e^{5/2}},
\end{aligned}\]
where the last inequality follows from the fact that $(1-x/n)^{n-1} \le e^{-x}$ for $n \ge x \ge 2$.
Combining this with lower estimate~\eqref{eq:P1_lb1_X},
\[ P_{\text{sep}} \ge 2 \gamma \nu(C')
    \ge 2 \frac{ \sqrt{n}(\alpha - 2\delta)}{\sqrt{2\pi}e^2 R}
    \frac{1-12e^{-5/2}/\sqrt{5}}{12}.
\]
Setting $\alpha = \delta + \epsilon/\sqrt{2n}$,
since $1-12e^{-5/2}/\sqrt{5} > 5/9$, we have that
\[
    P_{\text{sep}}
    \ge \frac{2\sqrt{n}(\epsilon/\sqrt{2n}-\delta)
    (1-12e^{-5/2}/\sqrt{5})}
    {12\sqrt{2\pi}e^{5/2}R}
    \ge \frac{ \epsilon - \delta \sqrt{2n} }{22\sqrt{\pi}e^{5/2} R}.
\]

Note that this bound holds under the assumption that $\alpha/\epsilon <
\sqrt{2/n}$, which for our choice of $\alpha$ is equivalent to the
assumption that $\epsilon > \delta\sqrt{2n}$.
However, this bound also
holds trivially for all $\epsilon \le \delta\sqrt{2n}$, and thus in fact
holds for all $\epsilon \ge 0$.

\paragraph{Case $n \le 5$}
In this case, note that $\xi/\sqrt{n} \ge 1$, so the sets $W$ and $Z$ are
the entire sphere.  Hence, $\nu(W^c) = \nu(Z^c) = 0$ and
$\nu(C') = \nu(C_\alpha) \ge \frac{1}{12}$.  Thus,
\[
    P_{\text{sep}} \ge 2\gamma\nu(C')
    \ge \frac{\epsilon - \delta\sqrt{2n}}{12\sqrt{\pi}e^{5/2} R}.
\]

We obtain the stated lemma by noting $\epsilon_0 \le \epsilon$.
\end{proof}

\begin{proof}\textbf{of Theorem~\ref{thm:rndmeasmain}\;}
Let $P_{\text{e}}$ denote the probability that there exists some $\x,\y \in \B_R^n$ with $\norm{\x - \y } > \epsilon$ and $\cA(\x) = \cA(\y)$. Our goal is to show that $P_{\text{e}} \le \eta$. Towards this end, let $U$ be a $\delta$-covering set for $\B_R^n$ with
$|U|\le (3R/\delta)^{n}$. By construction, for any $\x,\y \in \B_R^n$, there exist some $\w, \z \in U$ satisfying $\norm{\x-\w}\le\delta$ and $\norm{\y-\z}\le\delta$. In this case, if $\norm{\x - \y } > \epsilon$ then
\[
\norm{\w-\z} \ge \norm{\x-\y}-2\delta > \epsilon - 2\delta.
\]
Our goal is to upper bound the probability that there exists some $\w, \z \in U$ with $\norm{\w - \z} \ge \epsilon_0 = \epsilon - 2\delta$ and $\cA(\u) = \cA(\v)$ for some $\u\in B_\delta(\w)$ and $\v \in B_\delta(\z)$.  Said differently, we would like to bound the probability that there exists a $\w, \z \in U$ with $\norm{\w -\z} \ge \epsilon_0$ for which $B_\delta(\w)$ and $B_\delta(\z)$ are not separated by any of the $m$ hyperplanes.

Let $P_m(\w,\z)$ denote the probability that $B_\delta(\w)$ and $B_\delta(\z)$ are not separated by any of the $m$ hyperplanes for a fixed $\w, \z \in U$ with $\norm{\w - \z} \ge \epsilon_0$. Lemma~\ref{lem:singlemeas} controls this probability for a single hyperplane, yielding a bound of
\[
1 - P_{\text{sep}} \le 1 - \frac{ \epsilon_0  - \delta \sqrt{2 n} }{22 \sqrt{\pi} e^{5/2} R}.
\]
Since the $(\p_i, \q_i)$ are independent, we obtain
\begin{equation} \label{eq:Pmwz}
P_m(\w,\z) \le \left(1 - \frac{ \epsilon_0  - \delta \sqrt{2n} }{22 \sqrt{\pi} e^{5/2} R}\right)^m.
\end{equation}
Since we are interested in the event that there exists {\em any} $\w, \z \in U$ with $\norm{\w -\z} \ge \epsilon_0$ for which $B_\delta(\w)$ and $B_\delta(\z)$ are separated by {\em none} of the $m$ hyperplanes, we use the fact that there are at most $(3R/\delta)^{2n}$ such pairs $\w, \z$ and combine a union bound with~\eqref{eq:Pmwz} to obtain
\begin{equation} \label{eq:Peunion}
P_{\text{e}} \le \left(\frac{3R}{\delta} \right)^{2n} \left(1 - \frac{ \epsilon_0  - \delta \sqrt{2n} }{22 \sqrt{\pi} e^{5/2} R}\right)^m \le \exp \left( 2n \log \frac{3R}{\delta} - \frac{ \left(\epsilon_0  - \delta \sqrt{2n}\right) m }{22 \sqrt{\pi} e^{5/2} R}\right),
\end{equation}
which follows from $(1-x) \le e^{-x}$. Bounding the right-hand side of~\eqref{eq:Peunion} by $\eta$, we obtain
\begin{equation} \label{eq:mbound1}
2n \log \frac{3R}{\delta} - \frac{ \left(\epsilon_0  - \delta \sqrt{2n}\right) m }{22 \sqrt{\pi} e^{5/2} R} \le \log \eta.
\end{equation}
If we now make the substitutions $\epsilon_0 = \epsilon - 2\delta$ and $\delta = \epsilon/(4+ \sqrt{8n})$, then we have that $\epsilon_0 - \delta \sqrt{n} = \epsilon/2$
and thus we can reduce~\eqref{eq:mbound1} to
\begin{equation*}
2n \log \frac{3R(4+\sqrt{8n})}{\epsilon} - \frac{ \epsilon m }{44 \sqrt{\pi} e^{5/2} R} \le \log \eta.
\end{equation*}
By rearranging, we see that this is equivalent to
\begin{equation} \label{eq:mbound2}
    m \ge 44 \sqrt{\pi} e^{5/2} \frac{R}{\epsilon}
    \left( 2n\log\frac{3R(4+\sqrt{8n})}{\epsilon}
    + \log\frac{1}{\eta} \right).
\end{equation}
One can easily show that~\eqref{eq:sufficientm} implies~\eqref{eq:mbound2} for an appropriate choice of $C$.
\end{proof}

 
We now show that the result in Theorem~\ref{thm:rndmeasmain} is optimal in the
sense that {\em any} set of comparisons which can guarantee a uniform recovery
of all $\x \in \B^n_R$ to accuracy $\epsilon$ will require a number of comparisons on the same order as that required in Theorem~\ref{thm:rndmeasmain} (up to log factors).
\begin{theorem}
\label{thm:lowerbound}
For any configuration of $m$ (inhomogeneous)
hyperplanes in $\R^n$ dividing $\B^n_R$ into cells,
if $ m < \frac{2}{e}\frac{R}{\epsilon} n, $
then there exist two points $\x,\y\in \B^n_R$ in the same cell such that
$\norm{\x-\y} \ge \epsilon$.
\end{theorem}
\begin{proof}
We will use two facts. 
First, the number of cells
(both bounded and unbounded) defined by $m$ hyperplanes in $\R^n$ in general
position\footnote{For non-general position, this is an upper bound~\cite{partitionspace}.} is given by
\begin{equation} \label{eq:Fnbound}
F_n(m) = \sum_{i=0}^n \binom{m}{i} \le
    \left(\frac{em}{n}\right)^n < \left(\frac{2R}{\epsilon}\right)^n,
\end{equation}
where the second inequality follows from the assumption that $m < 2Rn/e\epsilon$.

Second, for any convex set $K$
we have the isodiametric inequality~\cite{giaquinta2010mathematical}:
where $\Diam(K) = \sup_{x,y\in K}\norm{x-y}$,
\begin{equation} \label{eq:isodia}
    \left( \frac{\Diam(K)}{2} \right)^n
    \frac{\pi^{n/2}}{\Gamma(n/2+1)} \ge
    \Vol(K),
\end{equation}
with equality when $K$ is a ball.
Since the entire volume of $\B^n_R$, denoted $\Vol(\B_R^n)$,
is filled by at most $F_n(m)$ non-overlapping
cells, there must exist at least one such cell $K_0$ with
\begin{equation} \label{eq:volbound}
    \Vol(K_0) \ge \frac{\Vol(\B_R^n)}{F_n(m)}
    = \frac{\pi^{n/2}}{\Gamma(n/2+1)} \frac{R^n}{F_n(m)}.
\end{equation}
Combining~\eqref{eq:isodia} with~\eqref{eq:volbound}, we obtain
\[
    \left( \frac{\Diam(K_0)}{2} \right)^n \ge \frac{R^n}{F_n(m)},
\]
which, together with~\eqref{eq:Fnbound}, implies that
\[
\Diam(K_0) \ge \frac{2R}{\sqrt[n]{F_n(m)}} > \epsilon.
\]
Thus there are vectors $\x, \y\in K_0$ such that $\| \x - \y \| > \epsilon$.
\end{proof}

 
\section{Stability in noise}\label{sec:noise}
So far, we have only considered the noise-free case.
In most practical applications, observations may be corrupted by noise.
We consider two scenarios;
in the first,
we make no assumption on the source of the errors and instead
show the paired comparison observations are stable with respect
to Euclidean distance.  That is, two signals that have similar sign
patterns are also nearby (and vice-versa).  One can view this as a
strengthening of the result in Theorem~\ref{thm:rndmeasmain}.
In the second case,
Gaussian noise is added prior
to the $\sign(\cdot)$ function in~\eqref{eq:measmodel}.
This is equivalent to the 
Thurstone model~\cite{thurstone1927law}
with the Probit (normal) link.

The results of this section apply without considering
any particular recovery method or algorithm and relate
to the number of paired comparisons which may be 
flipped due to the noise, not from reconstruction.
We address these concerns by introducing a practical algorithm
and associated guarantees in Section~\ref{sec:reconstr}.

Throughout the following, we denote by $d_H$ the Hamming distance, i.e., $d_H$
counts the fraction of comparisons which differ between two sets of observations, here denoted $\cA(\x)$ and $\cA(\y)$:
\begin{equation}\label{eq:hamming}
    d_H(\cA(\x), \cA(\y))
        := \frac{1}{m}\sum_{i=1}^m \frac{1}{2}|\cA_i(\x) -\cA_i(\y)|.
\end{equation}

 
\subsection{Stable embedding}\label{sec:concentrate}
Here we show that given enough comparisons
there is an approximate embedding of the preference space
into $\{-1,1\}^m$ via our model.
Theorem~\ref{thm:embedmain} states that
if $\x$ and $\y$ are sufficiently close, then the respective comparison
patterns $\cA(\x)$ and $\cA(\y)$ closely align.  In contrast with Theorem~\ref{thm:gaussnoise}, Theorem~\ref{thm:embedmain} is a purely geometric statement which makes no assumptions on any particular noise model.
Note also that Theorem~\ref{thm:embedmain} applies uniformly \emph{for all} $\x$ and $\y$.
\begin{theorem}\label{thm:embedmain}
Let $\eta,\,\zeta>0$ be given. Let $\cA(\x)$ denote the collection of $m$ observations defined as in Theorem~\ref{thm:rndmeasmain}.
There exist constants $C_1,c_1,C_2,c_2$ such that if
\begin{equation} \label{eq:embedm}
m \ge \frac{1}{2\zeta^2}
        \left(2n\log 
            \frac{3\sqrt{n}}{\zeta }
     + \log  \frac{2}{\eta} \right),
\end{equation}
then with probability at least $1-\eta$,
for all $\x,\y\in\B_R^n$ we have
\begin{equation} \label{eq:embedbound}
    C_1 \frac{\norm{\x-\y}}{R}
    - c_1 \zeta
    \le d_H(\cA(\x),\cA(\y))
    \le C_2 \frac{\norm{\x-\y}}{R} + c_2 \zeta.
\end{equation}
\end{theorem}

This result implies that the fraction of differences in the set of
observed comparisons between $\x$ and $\y$ will be constrained to within a
constant factor of the Euclidean distance, plus an additive error
approximately proportional to $1/\sqrt{m}$.  At first glance, this seems
worse than the result of Theorem~\ref{thm:rndmeasmain}, which suggests the
rate $1/m$.  However, Theorem~\ref{thm:embedmain} comes with much greater
flexibility in that Theorem~\ref{thm:rndmeasmain} only concerns the case
where $d_H(\cA(\x),\cA(\y)) = 0$.  As in Theorem~\ref{thm:rndmeasmain},
this result applies \emph{for all} $\x$ on the same randomly drawn
set of items.

This result is very reminiscent of Theorem~1.10
    in \cite{plan2014dimension} which concerns tessellations
    under uniform random affine hyperplanes generated according 
    to the Haar measure, rather than the particular Gaussian model we
    study.  Compared to that work, our result is much better in terms of
    the scaling of the lower bound on $m$ with respect to distortion
    $\zeta$ (they predict $1/\zeta^{12}$ while ours is $1/\zeta^2$).
    This is due to both our specific probabilistic model 
    and because we do not use the technique of ``lifting'' the problem
    to dimension $n+1$ in our analysis because it would incur additional
    distortion.  On the other hand, the result of \cite{plan2014dimension}
    is applicable to $\x$ lying in an arbitrary convex body
    whereas our result considers only $\x$ within a radius $R$ ball.

Unlike in Theorem~\ref{thm:rndmeasmain}, we are not aware whether the
relationship between $\zeta$ and $m$ given in Theorem~\ref{thm:embedmain}
is optimal.  This result is related to open questions
concerning Dvoretzky's theorem for embedding $\ell_2$ into $\ell_1$.
See the discussion of optimality
in Section~1.7 of \cite{plan2014dimension}
and Remark 1.6 of \cite{plan2013robust}.

In the context of a hypothetical recovery problem, suppose
$\x$ is a parameter of interest and $\y$ is an estimate produced by any
algorithm.  Then, \eqref{eq:embedbound} says that if we want to
recover $\x$ to within error $\epsilon$, the algorithm should
look for vectors $\y$ which have up to $O(\epsilon)$
incorrect comparisons.  Likewise, if a $\y$ can be found having up to
$O(\epsilon)$ comparison errors, we have the same $O(\epsilon)$
guarantee on the Euclidean error of the estimate.
In many cases, such as when errors are generated randomly,
Theorem~\ref{thm:rndmeasmain} would be inappropriate because finding a
$\y$ such that $d_H(\bar{\cA}(\x), \cA(\y)) = 0$ is likely to be
impossible.

To prove Theorem~\ref{thm:embedmain} we will require
the following Lemmas~\ref{lem:P0} and \ref{lemma:concentrate}.

\begin{lemma} \label{lem:P0}
Let $\w,\z\in \B^n_R$ be distinct and fixed, and let $\delta>0$ be given. Let $\cA(\x)$ denote the collection of $m$ observations defined as in Theorem~\ref{thm:rndmeasmain}, and let $B_{\delta}(\cdot)$ be defined as in Lemma~\ref{lem:singlemeas}. Denote by $P_0$ the probability that $B_\delta(\w)$ and $B_\delta(\z)$ are not separated by hyperplane $i$, i.e.,
     \[   P_0 = \pr \left[
            \forall \u\in B_\delta(\w),\forall \v\in
            B_\delta(\z) :
        \cA_i(\u) = \cA_i(\v) \right].
    \]
    Then
    \[
        1- P_0 \le \sqrt{\frac{2}{\pi}} \left( \frac{\norm{\w-\z}}{R} + \frac{\delta \sqrt{n}}{R}\right).
    \]
\end{lemma}
    \begin{proof}      We need an upper bound on
        \[ 1-P_0 = \pr \left[\cA_i(\u) \ne \cA_i(\v)
                \text{ for some }
                \u\in B_\delta(\w),
            \v\in B_\delta(\z) \right].
        \]
        Suppose for now that $\a$ is fixed and without loss of generality that $\a^T \w > \a^T \z$.  Then this probability is simply
        \begin{align*}
        \pr \left[ \a^T\v < \tau < \a^T \u
            \text{ for some } \u\in B_\delta(\w), \v\in B_\delta(\z) \right] & = \pr \left[ \min_{\v\in B_\delta(\z)} \a^T \v <
                \tau < \max_{\u\in B_\delta(\w)} \a^T \u \right] \\
                & \le \pr \left[ \a^T\z - \delta < \tau < \a^T\w + \delta \right],
        \end{align*}
        since by Cauchy--Schwarz we have
        \[
            \min_{\v\in B_\delta(\z)} \a^T \v \ge \a^T\z - \delta \quad \quad \text{and} \quad \quad
            \max_{\u\in B_\delta(\w)} \a^T \u \le \a^T\w + \delta.
        \]
        Thus, recalling that $\tau_i \sim \cN(0, R^2/n)$, from Lemma~\ref{lem:cdflower} we have
        \begin{align*}
        \pr \left[ \a^T\z - \delta < \tau < \a^T\w + \delta \right] & = \Phi\left(\frac{\a^T\w+\delta}{R/\sqrt{n}}\right)
                - \Phi\left(\frac{\a^T\z-\delta}{R/\sqrt{n}}\right) \\
                & \le \frac{1}{R} \sqrt{\frac{n}{2\pi}} \left(\a^T(\w-\z) + 2 \delta\right).
        \end{align*}
        Similarly, for $\a^T \w < \a^T \z$ we have
        \[
        \pr \left[ \a^T\w - \delta < \tau < \a^T\z + \delta \right] \le \frac{1}{R} \sqrt{\frac{n}{2\pi}} \left(\a^T(\z-\w) + 2 \delta\right).
        \]
        Combining these we have
        \begin{align*}
           1-P_0 & \le \int_{\S^{n-1}} \frac{1}{R} \sqrt{\frac{n}{2\pi}}
                \left( |\a^T(\w-\z)| +2\delta \right)
                \nu(\mathrm{d}\a)
            \\&= \frac{1}{R} \sqrt{\frac{n}{2\pi}} \int_{\S^{n-1}}
                |\a^T(\w-\z)| \,\nu(\mathrm{d}\a)
                 +\frac{2\delta}{R} \sqrt{\frac{n}{2\pi}}
            \\& = \frac{\sqrt{2n}}{R \pi}
                \frac{\Gamma(\frac{n}{2})}{\Gamma(\frac{n+1}{2})}
                \norm{\w-\z} + \frac{\delta}{R} \sqrt{\frac{2n}{\pi}},
        \end{align*}
        where the last equality is proven in Lemma~\ref{lem:sphereint}.
        The lemma then follows from the facts that
        $\frac{\Gamma(1/2)}{\Gamma(1)} = \sqrt{\pi}$ and
        $\frac{\Gamma(\frac{n}{2})}{\Gamma(\frac{n+1}{2})} \le
        \frac{2}{\sqrt{2n-1}} \le \sqrt{\frac{\pi}{n}}$ for $n\ge 2$
        \citep[(2.20)]{qi2010bounds}.
    \end{proof}

\begin{lemma}\label{lemma:concentrate}
Let $\w,\z\in \B_R^n$ be distinct and fixed, and let $\delta, \zeta > 0$ be given.
Let $\cA(\x)$ denote the collection of $m$ observations defined as in Theorem~\ref{thm:rndmeasmain}, and let $B_{\delta}(\cdot)$ be defined as in Lemma~\ref{lem:singlemeas}.
Then for all $\u \in B_{\delta}(\w)$ and $\v \in B_{\delta}(\z)$,
\[
    \frac{1}{22 e^{5/2}\sqrt{\pi}} \left(\frac{\norm{\w-\z}}{R } - \frac{\delta \sqrt{2n}}{R } \right) - \zeta
    \le d_H(\cA(\u),\cA(\v))
 \le \sqrt{\frac{2}{\pi}} \left( \frac{\norm{\w-\z}}{R} +
    \frac{\delta \sqrt{n}}{R} \right) + \zeta,
\]
with probability at least $1-\exp(-2\zeta^2m)$.
\end{lemma}

\begin{proof}Fix $\delta>0$ and let $\u\in B_\delta(\w), \v\in B_\delta(\z)$.
    Recall that the Hamming distance $d_H$ is a sum of independent
    and identically distributed Bernoulli random variables and
    we may bound it using Hoeffding's
    inequality.  Since our probabilistic upper and lower bounds must hold for
    all $\u,\v$ as described above, we introduce
    quantities $L_0$ and $L_1$ which represent two ``extreme cases'' of the
    Bernoulli variables:
    \[\bal
        L_0 & := \sup_{\u\in B_\delta(\w), \v \in B_\delta(\z)}
        \frac{1}{2m}\sum_{i=1}^m|\cA_i(\u)- \cA_i(\v)|
          \\ L_1 & := \inf_{\u\in B_\delta(\w), \v \in B_\delta(\z)}
          \frac{1}{2m}\sum_{i=1}^m|\cA_i(\u)- \cA_i(\v)|.
    \eal\]
    Then we have
    \[
        L_1 \le d_H(\cA(\u), \cA(\v)) \le L_0.
    \]
    Denote $P_0 = 1-\E L_0$ and $P_1 = \E L_1$, i.e.,
    \[\bal
        P_0 &= \pr \left[
            \forall \u\in B_\delta(\w),\forall \v\in
            B_\delta(\z) :
        \cA_i(\u) = \cA_i(\v) \right]
        \\ P_1 &= \pr \left[
            \forall \u\in B_\delta(\w),\forall \v\in
            B_\delta(\z) :
        \cA_i(\u) \ne \cA_i(\v) \right].
    \eal\]
    By Hoeffding's inequality,
    \[\bal
        \pr \left[ L_0 > (1-P_0)+\zeta \right] &\le \exp(-2m\zeta^2)
        \\ \pr \left[ L_1 < P_1-\zeta \right] &\le \exp(-2m\zeta^2).
    \eal\]
    Hence, with probability at least $1-2\exp(-2m\zeta^2)$,
    \[
        P_1 - \zeta \le d_H(\cA(\u), \cA(\v)) \le (1-P_0) + \zeta.
    \]
    The result follows directly from this combined with the facts that from Lemma~\ref{lem:singlemeas} we have
\[
    P_1 \ge
\frac{1}{22 e^{5/2}\sqrt{\pi}} \left(\frac{\norm{\w-\z}}{R} - \frac{\delta \sqrt{2n}}{R }\right),
\]
    and from Lemma~\ref{lem:P0} we have
    \[
        1-P_0 \le  \sqrt{\frac{2}{\pi}} \left( \frac{\norm{\w-\z}}{R} + \frac{\delta \sqrt{n}}{R}\right).
    \]
        ~\vspace{-0.42in}\\\mbox{}
\end{proof}

\begin{proof}\textbf{of Theorem~\ref{thm:embedmain}}\;
By Lemma~\ref{lemma:concentrate}, for any fixed pair $\w,\z\in\B_R^n$
we have
bounds on the Hamming distance that hold
with probability at least $1-2\exp(-2\zeta^2m)$, for all
$\u \in B_{\delta}(\w)$ and $\v \in B_{\delta}(\z)$.
Recall that the radius $R$ ball can be covered with a set $U$ of
radius $\delta$ balls with $|U| \le (3R/\delta)^n$.
Thus, by a union bound
we have that with probability at least $1-2(3R/\delta)^{2n}
\exp(-2\zeta^2m)$,
for {\em any} $\w,\z\in U$,
\[
    \frac{1}{22 e^{5/2}\sqrt{\pi}} \left(\frac{\norm{\w-\z}}{R } - \frac{\delta \sqrt{2n}}{R } \right) - \zeta
    \le d_H(\cA(\u),\cA(\v))
 \le \sqrt{\frac{2}{\pi}} \left( \frac{\norm{\w-\z}}{R} +
    \frac{\delta \sqrt{n}}{R} \right) + \zeta,
\]
for all $\u \in B_{\delta}(\w)$ and $\v \in B_{\delta}(\z)$.
Since $\norm{\x-\y}-2\delta\le\norm{\w-\z}\le\norm{\x-\y}+2\delta$, this implies that
\[
    \frac{1}{22 e^{5/2}\sqrt{\pi}} \left(\frac{\norm{\x-\y}-2\delta}{R } - \frac{\delta \sqrt{2n}}{R } \right) - \zeta
    \le d_H(\cA(\x),\cA(\y))
 \le \sqrt{\frac{2}{\pi}} \left( \frac{\norm{\x-\y}+2\delta}{R} +
    \frac{\delta \sqrt{n}}{R} \right) + \zeta,
\]
Letting $\delta = \zeta R/\sqrt{n}$ and setting $C_1, c_1, C_2, c_1$ appropriately\footnote{We set $C_1 = 1/22 e^{5/2} \sqrt{\pi}$ and $C_2 = \sqrt{2/\pi}$. We may set $c_1 = 1 + 1/11e^{5/2} \sqrt{\pi} + \sqrt{2/\pi}$ and $c_2 = 1 + 3\sqrt{2/\pi}$ to obtain constants that are valid for all $n$---improved values are possible for large $n$.} this reduces to~\eqref{eq:embedbound}.
Lower bounding the probability by $1-\eta$, we obtain
\[
    2(3\sqrt{n}/\zeta)^{2n}\exp(-2\zeta^2m) \le \eta.
\]
Rearranging yields~\eqref{eq:embedm}.
\end{proof}


\subsection{Gaussian noise}\label{sec:gaussnoise}
Here we aim to understand how the paired comparisons change
with the introduction of ``pre-quantization'' Gaussian noise. This will have
the effect of causing some comparisons to be erroneous, where the probability
of an error will be largest when $\x$ is equidistant from $\p_i$ and $\q_i$ and
will decay as $\x$ moves away from this boundary.

Towards this end, recall that the observation model in~\eqref{eq:measmodel00} can be reduced to the form
\begin{equation} \label{eq:measmodelz}
\cA_i(\x) = \sign(q_i) \quad \quad q_i := \a_i^T \x - \tau_i.
\end{equation}
In the noisy case, we will consider the observations
\begin{equation} \label{eq:measmodelzbar}
\bar{\cA}_i(\x) = \sign(\bar{q}_i) \quad \quad \bar{q}_i := \a_i^T \x - \tau_i + z_i = \bar{q}_i + z_i,
\end{equation}
where $z_i \sim\cN(0,\sigma_z^2)$.  Note that since $\norm{\a_i} = 1$, this model is equivalent to adding multivariate
Gaussian noise directly to $\x$ with covariance $\sigma_z^2 \mathbf{I}$.  For a fixed $\x$, we can then quantify the probability that $d_H(\cA(\x),\bar{\cA}(\x))$ is large via our next two results.
    We first consider the case of a single comparison in
    Lemma~\ref{lem:gaussnoise1}, then extend this
    to an arbitrary number of comparisons
    in Theorem~\ref{thm:gaussnoise}.
\begin{lemma}\label{lem:gaussnoise1}
    Suppose $n\ge 4$.
    Then $\pr[ \cA_i(\x) \neq \bar{\cA}_i(\x) ] \le \kappa_n(\sigma_z^2)$
    where $\kappa_n$ is defined by
\begin{equation}\label{eq:kappan}
    \kappa_n(\sigma_z^2)
    := \half\sqrt{\frac{\sigma_z^2}{\sigma_z^2 + 2 R^2/n + 4\norm{\x}^2/n}}
    \text{}\le \half\sqrt{\frac{1}{1+2R^2/(n\sigma_z^2)}}.
\end{equation}
\end{lemma}For clarity, we focus here on the $n\ge 4$ case.
    We consider the $n=2$ and $n=3$ cases separately because when
    $n\ge 4$ the probability distribution function of $\a_i^T\x$
    is well-approximated by a Gaussian function but not for $n < 4$.
    We give alternative expressions for $\kappa_n$ when
    $n=2$ and $n=3$ in Appendix~\ref{sec:noisedetails}.

\begin{theorem}
\label{thm:gaussnoise} Suppose
$n\ge 4$ and fix $\x\in\B^n_R$. Let $\cA(\x)$ and $\bar{\cA}(\x)$ denote the collection of $m$ observations defined as in~\eqref{eq:measmodelz} and~\eqref{eq:measmodelzbar} respectively, where the $\{(\p_i,\q_i)\}_{i=1}^m$ (and hence the $\{(\a_i,\tau_i)\}_{i=1}^m$) are generated as in Theorem~\ref{thm:rndmeasmain}.
Then,
\begin{equation} \label{eq:expdH}
    \E d_H( \cA(\x), \bar{\cA}(\x)) \le \kappa_n(\sigma_z^2)
    \text{}\le \half\sqrt{\frac{1}{1+2R^2/(n\sigma_z^2)}}.
\end{equation}
and
\begin{equation} \label{eq:taildH}
    \pr \left[ d_H( \cA(\x), \bar{\cA}(\x)) \ge \kappa_n(\sigma_z^2)
        + \zeta \right] \le \exp(-2m\zeta^2).
\end{equation}
where $\kappa_n$ is defined in~\eqref{eq:kappan}.
\end{theorem}
\begin{proof}
By Lemma~\ref{lem:gaussnoise1}, we have
that $\pr[ \cA_i(\x) \neq \bar{\cA}_i(\x) ]$ is bounded by
$\kappa_n(\sigma_z^2)$.
Since the comparisons are independent,
the expected number of sign mismatches is just the probability
of a sign flip just computed, which establishes~\eqref{eq:expdH}.
The tail bound in~\eqref{eq:taildH} is a simple consequence of Hoeffding's inequality.
\end{proof}

The bound \eqref{eq:expdH}
of Theorem~\ref{thm:gaussnoise} behaves as one
would expect.  If the variance $\sigma_z^2$ of the added
Gaussian noise is small, we predict that the expected fraction of
errors is also small.
To place this result in context, recall that $\tau_i \sim\cN(0, R^2/n)$. Suppose that $\sigma_z^2 = c_0 R^2/n$.  In this case one can bound $\kappa_n(\sigma_z^2)$ in~\eqref{eq:kappan} as

\[
\half\sqrt{ \frac{c_0}{c_0+6}} \le \kappa_n(\sigma_z^2) \le \half\sqrt{ \frac{c_0}{c_0+2} }.
\]
Intuitively, if $c_0$ is close to 1, meaning the noise variance is comparable to that of $\tau_i$, then we would expect to lose a significant amount of information about $\x$, in which case  $d_H( \cA(\x), \bar{\cA}(\x))$ could potentially be quite large.  In contrast, by letting $c_0$ grow small we can bound $\kappa_n(\sigma_z^2) \le \sqrt{c_0/8}$ arbitrarily close to zero.

It is instructive to consider Theorem~\ref{thm:gaussnoise}
next to Theorem~\ref{thm:embedmain},
which also predicts the fraction of sign mismatches up to an additive
constant which is proportional to $1/\sqrt{m}$.
\eqref{eq:expdH} and \eqref{eq:taildH} provide upper estimates
of the level of comparison errors which is unavoidable, regardless of
recovery technique.
If, in a particular application the noise is expected to be Gaussian, the
bound in~\eqref{eq:taildH} can be used in the lower half
of~\eqref{eq:embedbound} to create a recovery guarantee.
Note that since Theorem~\ref{thm:embedmain} is a uniform guarantee,
better estimates than this could be possible in the specific case of
Gaussian noise.
We explore Gaussian noise further in Section~\ref{subsec:estguarantee}.

\medskip
\begin{proof}\textbf{of Lemma~\ref{lem:gaussnoise1}}\;
The probability of a sign flip is given by
\[
     \pr \left[ q_i\bar q_i<0 \right] = \pr \left[ q_i < 0 ~~\text{and}~~ \bar q_i > 0 \right] + \pr \left[ q_i > 0 ~~\text{and}~~ \bar q_i < 0 \right].
\]
Note that if we set $r_i: = \a_i^T\x/\norm{\x} \in [-1,1]$, then we can write $q_i = r_i \norm{\x} - \tau_i$ and $\bar q_i = r_i \norm{\x} - \tau_i + z_i$ where the random variables $r_i$, $\tau_i$, and $z_i$ are independent.
Where $f_r(r_i)$ denotes the probability density functions for $r_i$
and recalling that $\tau_i \sim \cN(0,2R^2/n)$, we show in
Appendix~\ref{sec:noisedetails} using standard Gaussian tail bounds that
\begin{equation}
\pr[ q_i\bar q_i<0 ] \le \frac{1}{R} \sqrt{\frac{n}{\pi}} \int_{0}^1 \int_{-\infty}^\infty f_r(r_i) \exp\left(-\frac{(r_i \norm{\x}-\tau_i)^2}{2\sigma_z^2}-\frac{n \tau_i^2}{4R^2}\right) \der \tau_i \der r_i. \label{eq:Pqiqi}
\end{equation}
The remainder of the proof (given in Appendix~\ref{sec:noisedetails})
is obtained by bounding this integral. Note that in general, we have $\frac{1}{2}(r_i+1)\sim \text{Beta}((n-1)/2,(n-1)/2)$,
but $r_i$ is asymptotically normal with variance $1/n$
\cite{spruill2007asymptotic}.
For $n\ge 4$, we use the simple upper bound
\begin{align}
    f_r(r_i) &= \half \biggl[B\left(\frac{n-1}{2},
    \frac{n-1}{2}\right)\biggr]^{-1}
    \biggl(\frac{1+r_i}{2}
    \frac{1-r_i}{2}\biggr)^{(n-3)/2}            \notag
    \\ &\le \half \Biggl[\frac{\sqrt{2\pi}
        \frac{n-1}{2}^{(n-2)/2}\frac{n-1}{2}^{(n-2)/2}
    }{(n-1)^{n-1-1/2}}\Biggr]^{-1}
    \left(\frac{1-r_i^2}{4}\right)^{(n-3)/2}            \notag
    \\ &= \half \biggl[\frac{\sqrt{2\pi}}{2^{n-2}\sqrt{n-1}}\biggr]^{-1}
    \frac{1}{2^{n-3}}
    \exp(-(n-3)r_i^2/2)             \notag
    \\ &= \frac{\sqrt{n-1}}{4\sqrt{2\pi}}\exp(-(n-3)r_i^2/2)
    \le \frac{\sqrt{n}}{4\sqrt{2\pi}}\exp(-nr_i^2/8). \label{eq:ubfd}
\end{align}
This follows from the standard inequalities
$B(x,y) \ge \sqrt{2\pi}x^{x-1/2}y^{y-1/2}/(x+y)^{x+y-1/2}$
\citep[e.g.,][]{grenie2015inequalities}
and $1-x\le\exp(-x)$.
\end{proof}


\section{Estimation algorithm and guarantees}\label{sec:reconstr}

In the noise-free setting, given a set of comparisons $\cA(\x)$, we may
produce an estimate $\widehat{\x}$ by finding \emph{any} $\widehat{\x} \in \B^n_R$ satisfying $\cA(\widehat{\x}) = \cA(\x)$.
A simple approach is the following convex program:
\begin{equation}\label{eq:origproblem}
    \widehat{\x} = \argmin_{\w} \norm{\w}^2
     \quad\text{subject to }\quad  \cA_i(\x)(\a_i^T\w-\tau_i) \ge 0
     \quad\forall i\in[m].
\end{equation}
This is relatively easy to solve since the constraints are simple linear
inequalities and the feasible region is convex.
Note that~\eqref{eq:origproblem} is guaranteed to satisfy $\widehat{\x} \in \B^n_R$ since $\x \in \B^n_R$ and $\x$ is feasible, so that $\norm{\widehat{\x}} \le \norm{\x} \le R$.  In this case we may apply Theorem~\ref{thm:rndmeasmain} to argue that if $m$ obeys the bound in~\eqref{eq:sufficientm}, then $\norm{\widehat{\x} - \x} \le \epsilon$.

However, in most practical applications, observations are likely
to be corrupted by noise leading to inconsistencies.
Any errors in the observations
$\cA(\vx)$ would make strictly enforcing $\cA(\widehat{\x}) = \cA(\x)$ a questionable goal since, among other drawbacks, $\x$ itself would become infeasible.  In fact, in this case we cannot even necessarily guarantee that~\eqref{eq:origproblem} has {\em any} feasible solutions.
In the noisy case we instead use a relaxation inspired by the extended $\nu$-SVM of \cite{perez2003extension},
which introduces slack variables $\xi_i \ge 0$ and
is controlled by the parameter $\nu$.
Specifically, we denote by $\bar\cA(\x)$ the collection of (potentially)
corrupted measurements, and we solve
\begin{equation}\label{eq:nusvm}
    \begin{aligned}
    \underset{\widehat{\w}\in\R^{n+1} ,\boldsymbol{\xi}\in\R^m,\rho \in \R}
    {\text{minimize}}\quad
    & {-}\nu\rho + \frac{1}{m}\sum_{i=1}^m \xi_i
    \\ \text{ subject to }\quad &
        \bar\cA_i(\x)
        ([\a_i^T, -\tau_i]\widehat{\w}) \ge \rho-\xi_i, \quad \xi_i\ge0,  \quad \forall i \in [m],
        \\ & \norm{\widehat{\w}[1:n]}^2 \le {\textstyle\frac{2R^2}{1+R^2}},
        \quad \text{and} \quad
         \norm{\widehat{\w}}^2 = 2.
    \end{aligned}
\end{equation}
Finally, we set $\widehat\x = \widehat\w[1,\dots,n] / \widehat\w[n+1]$.
The additional constraint $\norm{\widehat{\w}[1:n]}^2
\le \frac{2R^2}{1+R^2}$ ensures that $\norm{\widehat{\x}}\le R$.
Note that an important difference between the extended $\nu$-SVM and \eqref{eq:nusvm}
is that there is no ``offset'' parameter to be optimized over.
That is, if we interpret $[\a_i,-\tau_i]$ as ``training examples,''
then $\w := [\x, 1]\in\R^{n+1}$ corresponds to
a \emph{homogeneous} linear classifier. Note that in the absence of comparison errors, setting $\nu = 0$,
we would have a feasible solution with $\xi_i = 0$.

Unfortunately, due to the norm equality constraint,
\eqref{eq:nusvm} is not convex and a unique
global minimum cannot be guaranteed, i.e., there may be multiple
solutions $\widehat\x$.  Nevertheless,
the following result shows that any local minimum will have certain desirable properties, and in the process also provides guidance on choosing the parameter $\nu$.  Combined with
our previous results, this also allows us to give recovery guarantees.

\begin{proposition}
    \label{prop:nusvm}
    At {\em any} local minimum $\widehat{\x}$ of \eqref{eq:nusvm}, we have $\frac{1}{m} |\{i:\xi_i>0\}|  \le \nu.$ If the corresponding $\rho > 0$, this further implies that $d_H(\cA(\widehat{\x}), \bar\cA(\x)) \le \nu$.
\end{proposition}

\begin{proof}
This proof follows similarly to that of Proposition 7.5
of~\cite{scholkopf2001learning}, except applied to the extended
$\nu$-SVM of~\cite{perez2003extension}
and with the removal of the hyperplane bias term.
Specifically, we first form the Lagrangian of~\eqref{eq:nusvm}:
\[\begin{aligned}
    L(\widehat\w,\boldsymbol{\xi},\rho,
    \boldsymbol{\alpha},\boldsymbol{\beta},\gamma,\delta)
    = -\nu\rho & + \frac{1}{m}\sum_i\xi_i
    - \sum_i(\alpha_i(\bar\cA_i(\x)[\a_i, -\tau_i]^T\widehat\w-\rho+\xi_i)
    + \beta_i\xi_i)
    \\ &
    + \gamma\biggl(\frac{2R^2}{1+R^2}-\norm{\widehat\w[1:n]}^2\biggr)
    - \delta(2-\norm{\widehat\w}^2).
\end{aligned}\]
We define the functions corresponding to the equality constraint
($h_1$) and inequality constraints ($g_i$ for $i\in[2m+1]$) as follows:
\begin{align*}
    h_1(\w,\boldsymbol{\xi},\rho) &:= (2-\norm{\w}^2),  \\
    g_{i}(\w,\boldsymbol{\xi},\rho) &:= \begin{cases}
        \bar\cA_i(\x)[\a_i,
        -\tau_i]^T\w-\rho+\xi_i  & i\in[1,m]  \\
        \xi_{i-m} &  i\in [m+1,2m]   \\
        -\Bigl(\frac{2R^2}{1+R^2}
            -\norm{\w[1:n]}^2\Bigr)
            & i = 2m+1. \\
    \end{cases}
\end{align*}

Consider the $n+m+2$ variables $(\widehat{\w},\boldsymbol{\xi},\rho)$.
The gradient corresponding to the equality constraint, $\boldsymbol{\nabla} \mathbf{h}_1$,
involves only the first $n+1$ variables.
Thus, there exists an $m+1$ dimensional subspace
$\mathcal{D}\subset\R^{n+m+2}$ where
for any $\mathbf{d} \in\mathcal{D}$, $\boldsymbol{\nabla} \mathbf{h}_1^T \mathbf{d} = 0$.
The gradients corresponding to the $2m+1$ inequality constraints are given in the $(2m+1) \times (n+m+2)$ matrix
\begin{align*}
    \mathbf{G} := \left[\begin{array}{c}
        \boldsymbol{\nabla}  \mathbf{g}_1^T      \\
        \vdots          \\
        \boldsymbol{\nabla}  \mathbf{g}_m^T      \\ \hline
        \boldsymbol{\nabla}  \mathbf{g}_{m+1}^T  \\
        \vdots          \\
        \boldsymbol{\nabla}  \mathbf{g}_{2m}^T   \\ \hline
        \boldsymbol{\nabla}  \mathbf{g}_{2m+1}^T \\
    \end{array}\right] = \left[
        \begin{array}{c|@{\;}cccc}
            \cdots & -1 & \cdots & 0 & 1\\
            \cdots & \vdots & \ddots & \vdots & \vdots  \\
            \leftarrow (n+1) \rightarrow & 0 & \cdots & -1 & 1 \\
            \cline{2-5} \text{irrelevant} & -1 & \cdots & 0 & 0\\
            \cdots & \vdots & \ddots & \vdots & \vdots \\
            \cdots & 0 & \cdots & -1 & 0\\
            \cline{2-5} \leftarrow\;\widehat{\w}[1:n]\;\rightarrow\;0
                & 0 & \cdots & 0 & 0
        \end{array}
    \right].
\end{align*}
Since there is a $\mathbf{d} \in\mathcal{D}$ such that $(\mathbf{G} \mathbf{d})[i] < 0$ for all $i$
(for example, $\mathbf{d}=[0,\dots,0|1,\dots,1,-1]$),
the Mangasarian--Fromovitz constraint qualifications hold
and we have the following
first-order necessary conditions for local minima
\citep[see e.g.,][]{bertsekas1999nonlinear},
\[
    \pd{L}{\rho} = -\nu + \sum\alpha_i
    \quad \implies \quad \sum\alpha_i = \nu
\]
and
\[
    \pd{L}{\xi_i} = \frac{1}{m} - \alpha_i - \beta_i = 0
    \quad \implies \quad \alpha_i+\beta_i=\frac{1}{m}.
\]
Since $\sum_{i=1}^m\alpha_i = \nu$, at most a fraction of $\nu$ can have
$\alpha_i=1/m$.  Now, any $i$ such that $\xi_i>0$ must
have $\alpha_i=1/m$ since by complimentary slackness, $\beta_i = 0$.
Hence, $\nu$ is an upper bound on the fraction of $\xi$ such that $\xi_i>0$.

Finally, note that if $\rho > 0$, then $\xi_i = 0$ implies
$\bar\cA_i(\x)([\a_i^T, -\tau_i]\widehat{\w}) \ge \rho-\xi_i > 0$.
Hence, the fraction of $\xi$ such that $\xi_i>0$ is an upper bound for $d_H(\cA(\widehat{\x}), \bar\cA(\x))$.
\end{proof}

\subsection{Estimation guarantees}\label{subsec:estguarantee}

We now show how the results of Theorems~\ref{thm:gaussnoise} and~\ref{thm:embedmain} can be combined with Proposition~\ref{prop:nusvm} to give recovery
guarantees on $\norm{\widehat{\x}-\x}$ when \eqref{eq:nusvm} is used
for recovery under realistic noisy observation models.
We consider three basic noise models.  In the first, an arbitrary (but small) fraction of comparisons are reversed.  We then consider the implication of this result in the context of two other noise models, one where Gaussian noise is added to either the underlying $\x$ or to the comparisons ``pre-quantization,'' that is, directly to $(\a_i^T\x-\tau_i)$, and another where the observations are generated using an arbitrary (but bounded) perturbation of $\x$. We will ultimately see that largely similar guarantees are possible in all three cases,
summarized in Table~\ref{tbl:estsummary}.

\begin{table}[h]\centering
    \caption{Summary of our estimation guarantees in
    this section, where each result holds separately with probability at
    least $1-\eta$ where $\eta,c_1,c_2,C_1,C_2$ are constants which may
    differ between results.
    }\label{tbl:estsummary}
\begin{tabular}{@{}clc@{}}\toprule
    Noise type
        & Guarantee on $\norm{\widehat{\x} - \x}/R$
        & Eq.{} \\ \midrule
    Adversarial, level $\kappa$
        & $\circ\le \frac{2}{C_1} \kappa
    + \frac{c_1}{C_1} \sqrt{ \frac{n \log(18m) + \log (2/\eta)}{2m}}$
        & \eqref{eq:upper_noise3}  \\
    i.i.d.\ Gaussian, $\sigma_z^2$
        &  $\circ\le \frac{\sqrt{2}}{C_1} \sqrt{\frac{n\sigma_z^2}{R^2}}  + \frac{1}{C_1} \sqrt{\frac{\log(1/\eta)}{2m}} + \frac{c_1}{C_1} \sqrt{ \frac{n \log(18m) + \log (2/\eta)}{2m}}$
        & \eqref{eq:Gaussnoisebound} \\
    Perturbations $x\mapsto x'$
        & $\circ\le \frac{2 C_2}{C_1} \frac{\norm{\x-\x'}}{R}
    + \frac{c_1 + 2c_2}{C_1} \sqrt{ \frac{n \log(18m) + \log (2/\eta)}{2m}}$
        & \eqref{eq:arbnoisebound} \\ \bottomrule
\end{tabular}
\end{table}

In our analysis of all three settings, we will use the fact that from the lower bound of Theorem~\ref{thm:embedmain} we have
\begin{equation} \label{eq:upper_noise}
    \frac{\norm{\widehat{\x} - \x}}{R}
    \le \frac{d_H(\cA(\x), \bar\cA(\x)) + c_1 \zeta}{C_1}
\end{equation}
with probability at least $1-\eta$ provided that $m$ is sufficiently large, e.g., by taking
\begin{equation} \label{eq:mboundembed}
m = \frac{1}{2\zeta^2}
    \left(2n\log
            \frac{3\sqrt{n}}{\zeta }
      + \log  \frac{2}{\eta} \right).
\end{equation}
Note that by setting $\beta = \frac{9n}{\zeta^2} \left(\frac{2}{\eta}\right)^{1/n}$, we can rearrange~\eqref{eq:mboundembed} to be of the form
\[
18m \left(\frac{2}{\eta}\right)^{1/n} = \beta \log \beta,
\]
which implies that
\[
\beta = \frac{18m (2/\eta)^{1/n}}{W \left(18m (2/\eta)^{1/n} \right)},
\]
where $W(\cdot)$ denotes the Lambert $W$ function. Using the fact that $W(x) \le \log(x)$ for $x \ge e$ and substituting back in for $\beta$, we have
\[
\zeta \le \sqrt{ \frac{n \log(18m) + \log (2/\eta)}{2m}}
\]
under the mild assumption that $m \ge \frac{e}{18} (\frac{\eta}{2})^{1/n}$. Substituting this in to~\eqref{eq:upper_noise} yields
\begin{equation} \label{eq:upper_noise2}
    \frac{\norm{\widehat{\x} - \x}}{R}
    \le \frac{d_H(\cA(\x), \bar\cA(\x))}{C_1}
    + \frac{c_1}{C_1} \sqrt{ \frac{n \log(18m) + \log (2/\eta)}{2m}}.
\end{equation}
We use this bound repeatedly below.

\paragraph{Noise model 1.}

In the first noise model, we suppose that an adversary is allowed to arbitrarily flip a fraction $\kappa$ of measurements, where we assume $\kappa$ is known (or can be bounded). This would seem to be a challenging setting, but in fact a guarantee under this model follows immediately from the lower bound in Theorem~\ref{thm:embedmain}.  Specifically, suppose that $\cA(\x)$ represents the noise-free comparisons, and we receive instead $\bar\cA(\x)$, where $d_H(\cA(\x),\bar\cA(\x)) \le \kappa$.

Consider using \eqref{eq:nusvm} to produce an $\widehat{\x}$
setting $\nu=\kappa$.
If $\widehat{\x}$ is a local minimum for \eqref{eq:nusvm} with $\rho > 0$,
Proposition~\ref{prop:nusvm} implies that $d_H( \cA(\widehat{\x}), \bar\cA(\x) ) \le \kappa$. Thus, by the triangle inequality,
\[
    d_H( \cA(\widehat{\x}), \cA(\x) )
    \le d_H( \cA(\widehat{\x}), \bar\cA(\x) )
    + d_H(\cA(\x), \bar\cA(\x) ) \le 2\kappa.
\]
Plugging this into~\eqref{eq:upper_noise2} we have that with probability at least $1-\eta$
\begin{equation} \label{eq:upper_noise3}
    \frac{\norm{\widehat{\x} - \x}}{R}
    \le \frac{2 \kappa}{C_1}
    + \frac{c_1}{C_1} \sqrt{ \frac{n \log(18m) + \log (2/\eta)}{2m}}.
\end{equation}

We emphasize the power of this result---the adversary may flip
not merely a random fraction of comparisons,
but an \emph{arbitrary} set of comparisons.
Moreover, this holds uniformly for all $\x$ and $\widehat{\x}$
simultaneously (with high probability).

\paragraph{Noise model 2.}

Here we model errors as being generated by adding i.i.d.\ Gaussian before the $\sign(\cdot)$
function, as described in Section~\ref{sec:gaussnoise}, i.e.,
\[
    \bar\cA_i(\x) = \sign(\a_i^T\x - \tau_i + z_i),
\]
where $z_i \sim \mathcal{N}(0,\sigma_z^2)$.
Note that this model is equivalent to the Thurstone model of comparative
judgment \cite{thurstone1927law}, and causes a predictable probability of
error depending the geometry of the set of items.  Specifically, comparisons which are
``decisive,'' i.e., whose hyperplane lies far from $\x$, are unlikely to
be affected by this noise.  Conversely, comparisons which are nearly
even are quite likely to be affected.

Under the random observation model considered in this paper, by Theorem~\ref{thm:gaussnoise} we have that, with probability at least $1-\eta$,
\[
    d_H(\cA(\x), \bar\cA(\x)) \le \kappa_n(\sigma_z^2) + \sqrt{\frac{\log (1/\eta)}{2m}},
\]
where
\[
\kappa_n(\sigma_z^2) = \sqrt{\frac{\sigma_z^2}{\sigma_z^2 + 2 R^2/n + 4\norm{\x}^2/n}} \le \sqrt{\frac{n\sigma_z^2}{2 R^2}}.
\]
We now assume that $\widehat{\x}$ is a local minimum of \eqref{eq:nusvm}
with $\nu = \kappa_n(\sigma_z^2)$ such that $\rho > 0$.
By the triangle inequality and Proposition~\ref{prop:nusvm},
\[
    d_H(\cA(\widehat{\x}), \cA(\x))
    \le d_H(\cA(\widehat{\x}), \bar\cA(\x)) + d_H(\cA(\x),\bar\cA(\x))
    \le 2 \sqrt{\frac{n\sigma_z^2}{2 R^2}}  + \sqrt{\frac{\log(1/\eta)}{2m}}.
\]
Combining this with~\eqref{eq:upper_noise2}, we have that with probability at least $1-2\eta$,
\begin{equation} \label{eq:Gaussnoisebound}
\frac{\norm{\widehat{\x} - \x}}{R}
    \le \frac{\sqrt{2}}{C_1} \sqrt{\frac{n\sigma_z^2}{R^2}}  + \frac{1}{C_1} \sqrt{\frac{\log(1/\eta)}{2m}} + \frac{c_1}{C_1} \sqrt{ \frac{n \log(18m) + \log (2/\eta)}{2m}}.
\end{equation}

We next consider an alternative perspective on this model. Specifically, suppose that our observations are generated via
\[
    \bar\cA_i(\x) = \cA_i(\x'_i)   \quad\text{where}\quad
    \x'_i = \x + \z_i,
\]
where $\z_i\sim\N(0,\sigma_z^2I)$.  Note that we can write this as
\[
    \cA_i(\x'_i) = \a_i^T(\x + \z_i) - \tau_i
        = \a_i^T\x - \tau_i + \a_i^T \z_i.
\]
Since $\norm{\a_i}=1$, $\a_i^T \z_i \sim \mathcal{N}(0,\sigma_z^2)$, and thus this is equivalent to the model described above.  Thus, we can also interpret the above results as applying when each comparison is generated using a ``misspecified'' version of $\x$ which has been perturbed by Gaussian noise.  Moreover, note that
\[
\E \norm{\x - \x'_i}^2 = \E \norm{\z_i}^2 = n \sigma_z^2,
\]
in which case we can also express the bound in~\eqref{eq:Gaussnoisebound} as
\begin{equation} \label{eq:Gaussnoisebound2}
\frac{\norm{\widehat{\x} - \x}}{R}
    \le \frac{\sqrt{2}}{C_1} \sqrt{\frac{\E \norm{\x - \x'_i}^2}{R^2}}  + \frac{1}{C_1} \sqrt{\frac{\log(1/\eta)}{2m}} + \frac{c_1}{C_1} \sqrt{ \frac{n \log(18m) + \log (2/\eta)}{2m}}.
\end{equation}
Thus, a small Gaussian perturbation of $\x$ in the comparisons will result in an increased recovery error roughly proportional to the (average) size of the perturbation.

Note that in establishing this result we apply Theorem~\ref{thm:gaussnoise}, and so in contrast to our first noise model, here the result holds with high probability for a fixed $\x$ (as opposed to being uniform over all $\x$ for a single choice of $\cA$).

\paragraph{Noise model 3.}

In the third noise model, we assume the comparisons
are generated according to
\[
    \bar\cA(\x) = \cA(\x'),
\]
where $\x'$ represents an arbitrary perturbation of $\x$.  Much like in the previous model, comparisons which are ``decisive'' are not likely to be affected by this kind of noise, while comparisons which are nearly even are quite likely to be affected. Unlike the previous model, our results here make no assumption on the distribution of the noise and will instead use the upper bound in Theorem~\ref{thm:embedmain} to establish a uniform guarantee that holds (with high probability) simultaneously for all choices of $\x$ (and $\x'$). Thus, in this model our guarantees are quite a bit stronger.

Specifically, we use the fact that from the upper bound of Theorem~\ref{thm:embedmain},  with probability at least $1-\eta$ we simultaneously have~\eqref{eq:upper_noise} and
\[
    d_H(\cA(\x),\bar\cA(\x)) \le C_2\frac{\norm{\x-\x'}}{R}
    + c_2 \zeta
    =: \kappa.
\]
We again use \eqref{eq:nusvm} with $\nu = \kappa$ and Proposition~\ref{prop:nusvm} to produce an
estimate $\widehat{\x}$ satisfying $d_H( \cA(\widehat{\x}), \bar\cA(\x) ) \le \kappa$.
Again using the triangle inequality, we have
\[
    d_H( \cA(\widehat{\x}), \cA(\x) )
    \le d_H( \cA(\widehat{\x}), \bar\cA(\x) )
    + d_H(\cA(\x), \bar\cA(\x) ) \le 2\kappa.
\]
Combining this with~\eqref{eq:upper_noise}
we have
\[
    \frac{\norm{\widehat{\x} - \x}}{R}
    \le
    \frac{2 \kappa + c_1 \zeta}{C_1}
    =  \frac{2 C_2}{C_1} \frac{\norm{\x-\x'}}{R}
    + \frac{c_1 + 2c_2}{C_1} \zeta.
\]
Substituting in for $\zeta$ as in~\eqref{eq:upper_noise2} yields
\begin{equation} \label{eq:arbnoisebound}
    \frac{\norm{\widehat{\x} - \x}}{R}
    \le
    \frac{2 C_2}{C_1} \frac{\norm{\x-\x'}}{R}
    + \frac{c_1 + 2c_2}{C_1} \sqrt{ \frac{n \log(18m) + \log (2/\eta)}{2m}}.
\end{equation}
Contrasting the result in~\eqref{eq:arbnoisebound} with that in~\eqref{eq:Gaussnoisebound2}, we note that up to constants, the results are essentially the same.  This is perhaps somewhat surprising since~\eqref{eq:arbnoisebound} applies to arbitrary perturbations (as opposed to only Gaussian noise), and moreover,~\eqref{eq:arbnoisebound} is a uniform guarantee.

 
\subsection{Adaptive estimation}\label{sec:adaptive}
Here we describe a simple extension to our previous (noiseless)
theory and show that
if we modify the mean and variance of the sampling distribution of items
over a number of stages, we can localize \emph{adaptively} and produce an
estimate with many fewer comparisons than possible in a non-adaptive
strategy.  We assume $t$ stages ($t=1$ for the non-adaptive approach).  At
each stage $\ell\in[t]$ we will attempt to produce an estimate
$\widehat \x^{\ell}$
such that $\|\x-\widehat\x^{\ell}\|\le\epsilon_\ell$
where $\epsilon_\ell = R_\ell/2 = R 2^{-\ell}$,
then recentering to our previous estimate and dividing the
problem radius in half.  In stage $\ell$,
each $\p_i,\q_i\sim\cN(\widehat\x, 2R_\ell^2/n \mat I)$.
After $t$ stages we will have
$\|\x - \widehat \x^t\| \le R 2^{-t} =: e_t$
with probability at least $1-t\eta$.
\begin{proposition}\label{prop:adaptive}
Let $\epsilon_t, \eta > 0$ be given. Suppose that $\x\in\B^n_R$ and that
$m$ total comparisons are obtained following the adaptive scheme where
\[
    m \ge 2C \log_2\left(\frac{2R}{\epsilon_t}\right)
    \left(n\log 2 \sqrt{n} + \log\frac{1}{\eta}\right),
\]
where $C$ is a constant.
Then with probability at least $1-\log_2(2R/\epsilon_t)\eta$,
for any estimate $\widehat\x$ satisfying $\cA(\widehat\x) = \cA(\x)$,
\[
    \norm{\x-\widehat\x} \le \epsilon_t.
\]
\end{proposition}
\begin{proof}
The adaptive scheme uses
$t = \ceil{\log_2(R/\epsilon_t)}\le\log_2(2R/\epsilon_t)$ stages.
Assume each stage is allocated $m_\ell$ comparisons.
By Theorem~\ref{thm:rndmeasmain}, localization at each
stage $\ell$ can be accomplished with high probability when
\[\begin{aligned}
    m_\ell &\ge C \frac{R_\ell}{\epsilon_\ell}
    \left( n\log\frac{R_\ell \sqrt{n}}{\epsilon_\ell}
    + \log\frac{1}{\eta} \right)
    = 2C
    \left( n\log2\sqrt{n} + \log\frac{1}{\eta} \right).
\end{aligned}\]
This condition is met by giving an equal number of
comparisons to each stage, $m_\ell = \flr{m/t}$.
Each stage fails with probability $\eta$.  By a union bound,
the target localization fails with probability at most $t\eta$.
Hence, localization succeeds with probability at least $1-t\eta$.
\end{proof}

Proposition~\ref{prop:adaptive} implies
$m_\text{adapt} \asymp (n\log n) \log_2(R/\epsilon_t) $
comparisons suffice to estimate $\x$ to within $\epsilon_t$.
This represents an exponential improvement
in terms of number of total comparisons
as a function of the target accuracy, $\epsilon_t$,
as compared to a lower bound on the number of required comparisons,
$m_\text{lower} := 2nR/(e\epsilon_t)$
for \emph{any} non-adaptive strategy (recall Theorem~\ref{thm:lowerbound}).

Note that this result holds in the noise-free setting, but can be
generalized to handle noisy settings via the approaches discussed in
Sections~\ref{sec:noise} and~\ref{sec:reconstr}.
While we cannot hope to achieve the same exponential improvement in the
general case, we expect rates better than those of passive methods to be
possible for certain noise levels.
 
\section{Simulations}\label{sec:experiments}

In this section we perform a range of synthetic experiments to demonstrate our approach.

\subsection{Effect of varying \texorpdfstring{$\sigma^2$}{the variance}}
In Fig.~\ref{fig:sigmachoice}, we let $\x\in\R^{\{2,3,4,8\}}$ with $\norm{\x}=R=1$.
We vary $\sigma^2$ and perform 1500 trials,
each with $m=100$ or $m = 200$ pairs of points drawn according to $\N(0,\sigma^2I)$.  To isolate the impact of $\sigma^2$, we consider the case where our observations are noise-free, and use~\eqref{eq:origproblem} to recover $\widehat{\x}$.
As predicted by the theory,
localization accuracy depends on the parameter $\sigma$, which controls
the distribution of the hyperplane thresholds.
Intuitively, if $\sigma$ is too small, the hyperplane boundaries
concentrate closer to the origin and do not localize points with large norm
well.  On the other hand, if $\sigma$ is too large, most hyperplanes lie far
from the target $\x$.  The sweet spot which allows uniform localization
over the radius $R$ ball exists around $\sigma^2\approx 2R^2/n$ here.

\def\fw{5.5in}
\begin{figure}[htbp]
    \centering \begin{tabular}{cc}
        \includegraphics[width=2in]{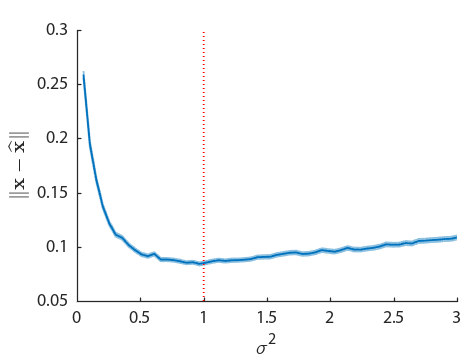}
        & \includegraphics[width=2in]{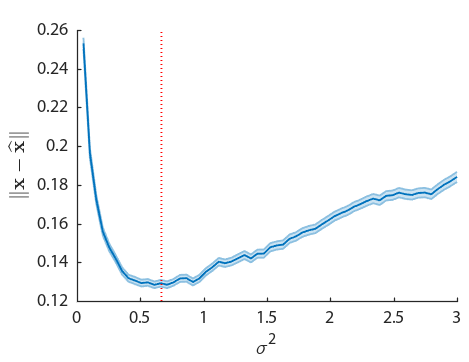}
        \\ (a) & (b) \\
        \includegraphics[width=2in]{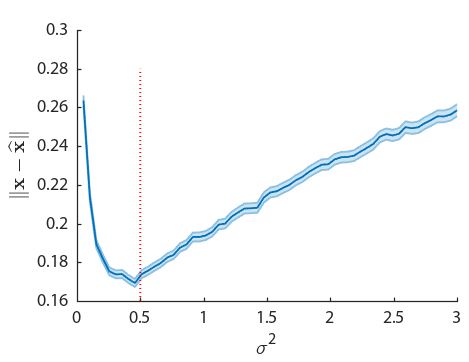}
        & \includegraphics[width=2in]{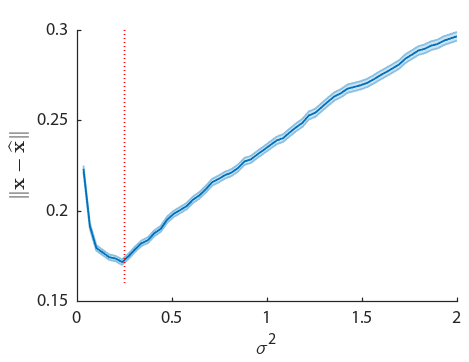}
        \\ (c) & (d) \\
    \end{tabular}
    \caption{Mean error norm $\norm{\x-\widehat{\x}}$
    as $\sigma^2$ varies for dimensions (a) 2, (b) 3, (c) 4, 
and (d) 8 for (a--c) 100 comparisons and (d) 200 comparisons.}
    \label{fig:sigmachoice}
\end{figure}

\subsection{Effect of noise}
Here we experiment with noise as discussed in Sections~\ref{sec:noise} and~\ref{sec:reconstr}, and use the optimization program~\eqref{eq:nusvm} for recovery.  To approximately
solve this non-convex problem, we use the linearization procedure
described in \cite{perez2003extension}.  Specifically, over a number of
iterations $k$, we repeatedly solve the sub-problem
\begin{align*}
    \underset{\rho\in\R,\boldsymbol{\xi}
    \in\R^m,\widehat{\w}^{(k)}\in\R^{n+1}}
    {\text{minimize}}\quad
    & {-}\nu\rho + \frac{1}{m}\sum_{i=1}^m \xi_i
    \\ \text{ subject to }\quad &
    \bar\cA_i(\x)
    ([\a_i^T, -\tau_i]\widehat{\w}^{(k)}) \ge \rho-\xi_i,
    \quad \xi_i\ge0,  \quad \forall i \in [m],
    \\ &
    \widehat{\w}^{(k)T}\widetilde{\w}^{(k)} = 2
\end{align*}
where we set $\widetilde{\w}^{(k+1)} \leftarrow
\chi\widetilde{\w}^{(k)}  + (1-\chi)\widehat{\w}^{(k)}$
with $\chi = 0.7$.  After sufficient iterations, if
$\widetilde{\w}^{(k)} \approx \widehat{\w}^{(k)}$
then \eqref{eq:nusvm} is approximately solved.
This is a linear program and it can be easily verified using the KKT
conditions that $|\{i : \xi_i > 0\}| \le m \nu$.
Thus in practice, this property will always be satisfied after
each iteration.

We also emphasize that the error bounds in Section~\ref{sec:reconstr} rely on the fact from Proposition~\ref{prop:nusvm} that $d_H(\cA(\widehat{\x}), \bar\cA(\x)) \le \nu$, {\em provided that the solution results in a $\rho > 0$}. Unfortunately, we cannot guarantee that this will always be the case. Empirically, we have observed that given a certain noise level quantified by $d_H(\cA(\x), \bar\cA(\x)) = \kappa$, we are more likely to observe $\rho \le 0$ when we aggressively set $\nu = \kappa$.  By increasing $\nu$ somewhat this becomes much less likely. As a rule of thumb, we set $\nu=2\kappa$. We note that while in our context this choice is purely heuristic, it has some theoretical support in the $\nu$-SVM literature
\citep[e.g., see Proposition 5 of][]{nusvmtutorial}.

We consider the following
noise models; {\em (i)} \emph{Gaussian}, where we add pre-quantization
Gaussian noise as in Section \ref{sec:gaussnoise},
{\em (ii)} \emph{logistic}, in which pre-quantization logistic noise is added,
{\em (iii)} \emph{random}, where a uniform random $\nu/2$ fraction of
comparisons are flipped, and {\em (iv)} \emph{adversarial}, where we flip the $\nu/2$
fraction of comparisons whose hyperplanes lie farthest from the
ideal point.
The logistic noise assumption is commonly studied in paired
    comparison literature and is equivalent to the
    Bradley--Terry model~\cite{bradley1952rank}. 
    Although we do not have theory for this case, we expect
    the logistic noise case to behave similarly to the Gaussian case.
    In both the Gaussian and logistic cases, we sweep the added
    noise variance, then for each variance plot against
    the mean number of errors induced as the x-axis.
In each case, we set $n=5$ and generate $m=1000$ pairs of points and a random $\x$ with
$\norm{\x} = 0.7$.

The mean and median recovery error
$\norm{\widehat{\x}-\x}$ and the fraction of violated comparisons
$d_H(\cA(\widehat{\x}), \cA(\x))$ are plotted over 100 independent
trials with varying number of comparison errors in
Figs.~\ref{fig:Gaussian:a32}--\ref{fig:adversarial:a30}.
In  the Gaussian noise, logistic noise, and uniform random comparison flipping cases,
the actual fraction of comparison errors of the estimate is on average much smaller than
our target $\nu$.  This is also seen in the adversarial case
(Fig.~\ref{fig:adversarial:a30}) for smaller levels of error.
However, at a high fraction of error (greater than about 17\%) the error
(both in terms of Euclidean norm and fraction of incorrect comparisons)
grows rapidly.  This illustrates a limitation to the approach of using
slack variables as a relaxation to the 0--1 loss.
We mention that in this regime, the recovery approach of~\eqref{eq:nusvm}
frequently yields $\rho \le 0$, to which our theory does not apply.
Here, the recovery error counter-intuitively \emph{decreases} with increasing comparison flips. This scenario, with a large number of erroneous comparisons, represents a very
difficult situation in which any tractable recovery strategy would
likely struggle.  A possible direction for future work would be to
make \eqref{eq:nusvm} more robust to such large outliers.

\begin{figure}[htbp]\centering 
\includegraphics[width=\fw]{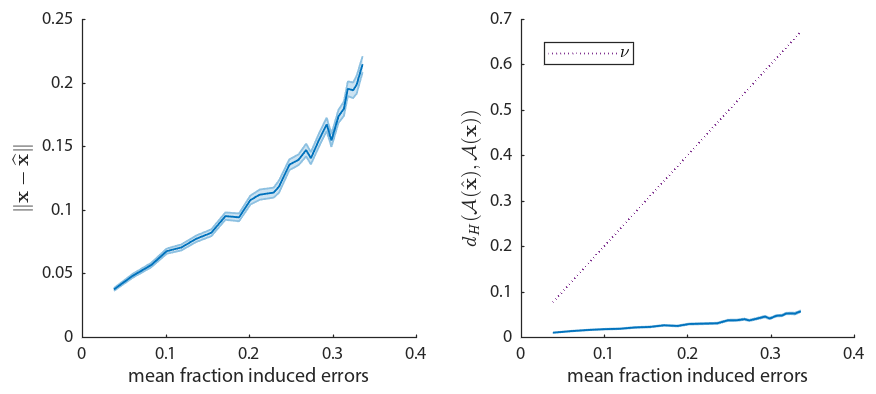}
    \caption{Mean estimation error and average fraction comparison errors
        when adding pre-quantization
        Gaussian noise, sweeping
        the variance and plotting against the average number of
        induced errors for each variance.}
    \label{fig:Gaussian:a32}
\end{figure}
\begin{figure}[htbp]\centering 
\includegraphics[width=\fw]{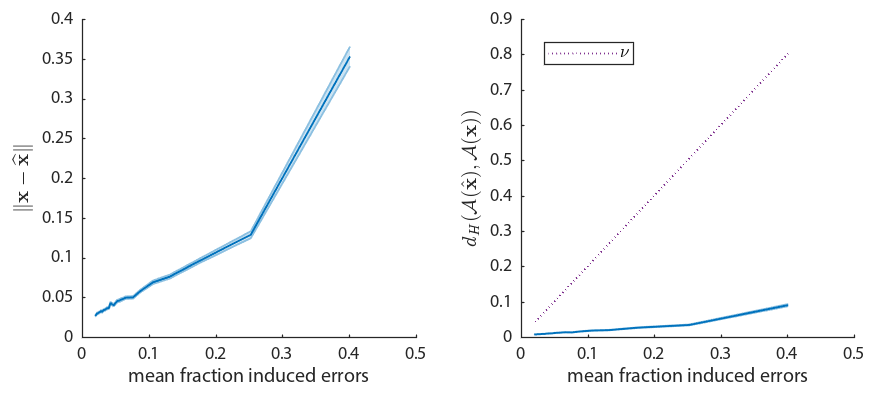}
    \caption{Mean estimation error and fraction comparison errors
        when adding pre-quantization
        logistic noise, sweeping
        the variance and plotting against the average number of
        induced errors for each variance.}
    \label{fig:logistic:a41}
\end{figure}
\begin{figure}[htbp]\centering 
\includegraphics[width=\fw]{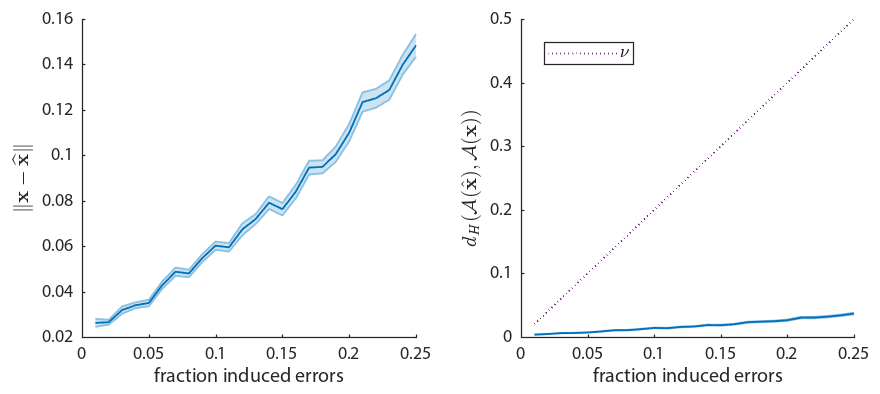}
    \caption{Mean estimation error and fraction comparison errors when introducing uniform random comparison errors.}
    \label{fig:random:a31}
\end{figure}
\begin{figure}[htbp]\centering 
\includegraphics[width=\fw]{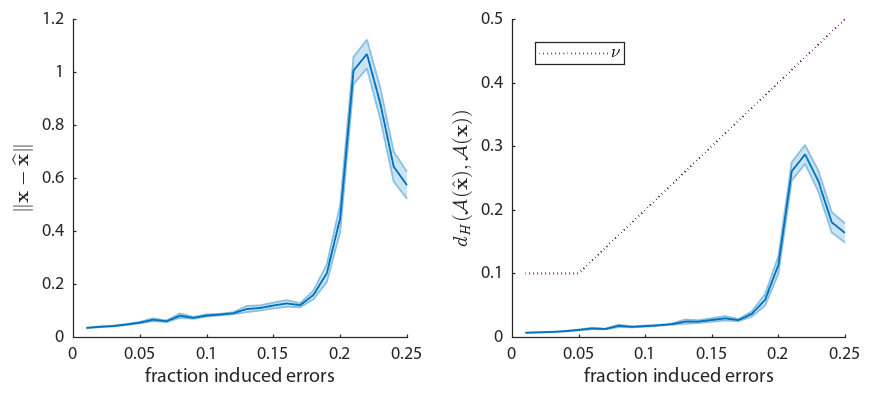}
    \\\medskip
    \includegraphics[width=3.6in]{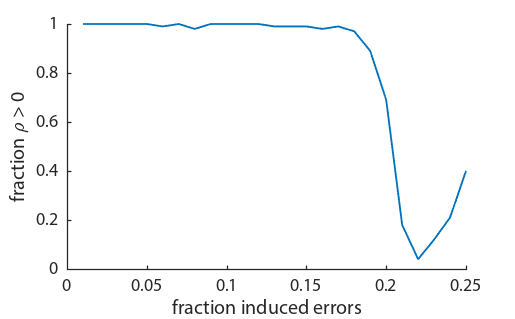}
    \caption{(Top) Mean estimation error and fraction comparison errors
        when flipping the farthest comparisons.
        (Bottom) Fraction of trials with $\rho > 0$.
    } 
    \label{fig:adversarial:a30}
\end{figure}

\subsection{Adaptive comparisons}
In Fig.~\ref{fig:compare1}, we show the effect of
varying levels of adaptivity, starting with the completely non-adaptive
approach up to using 10 stages where we progressively
re-center and re-scale the hyperplane offsets.
In each case, we generate
$\x\in\R^3$ where $\norm{\x} = 0.75$ and choosing the direction randomly.
The total number of comparisons are held fixed and are split as equally as possible
among the number of stages (preferring earlier stages when rounding).
We set $\sigma^2=R=1$ and plot the average over
700 independent trials.
As the number of stages increases, performance worsens
if the number of comparisons are kept small due to bad localization
in the earlier stages.  However, if the number of total
comparisons is sufficiently large,
an exponential improvement over non-adaptivity is possible.

\begin{figure}[htbp]
    \centering 
\includegraphics[width=4.3in]{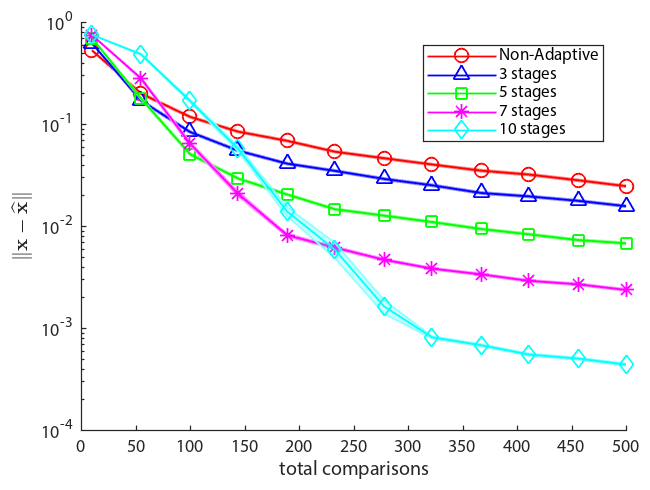}
    \caption{Mean error norm $\norm{\x-\widehat\x}$ versus
        total comparisons for a
        sequence of experiments with varying number of adaptive stages.}
    \label{fig:compare1}
\end{figure}

\subsection{Adaptive comparisons with a fixed non-Gaussian data set}
\label{sec:exper3}
In Fig.~\ref{fig:compare2}, we demonstrate the effect of adaptively choosing
item pairs from a fixed synthetic data set over four stages versus choosing
items non-adaptively, i.e., without attempting to estimate the signal during
the comparison collection process.  We first generated 10,000 items uniformly
distributed inside the 3-dimensional unit ball and a vector $\x\in\R^3$ where
$\norm{\x} = 0.4$.  In both cases, we generate pairs of Gaussian points and
choose the items from the fixed data set which lie closest to them.  In the
adaptive case over four stages, we progressively re-center and re-scale the
generated points; the initial $\sigma^2$ is set to the variance of the data set
and is reduced dyadically after each stage.  The total number of comparisons
is held fixed and is split as equally as possible among the number of stages
(preferring later stages when rounding).  We plot the mean error over 200
independent data set trials.

\begin{figure}[htbp]
    \centering 
\includegraphics[width=4.3in]{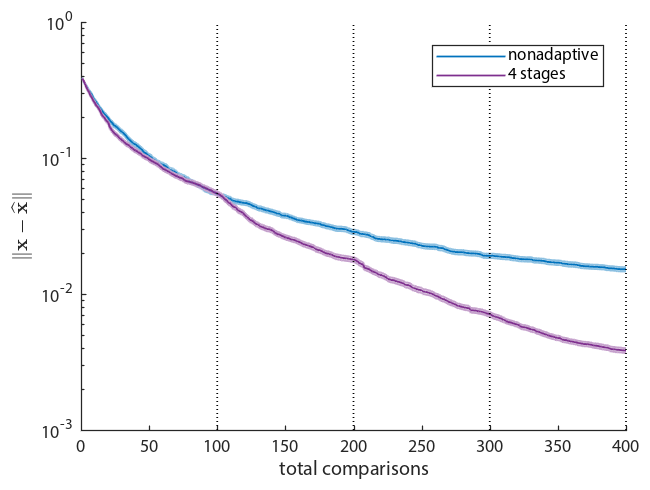}
    \caption{Mean error norm $\norm{\x-\widehat \x}$ versus
        total comparisons for nonadaptive and adaptive selection.
        Dotted lines denote stage boundaries.}
    \label{fig:compare2}
\end{figure}

 
\section{Discussion}\label{sec:discussion}
We have shown that given the ability to generate item pairs according to a Gaussian distribution with a particular variance, it is possible
to estimate a point $\x$ satisfying $\norm{\x} \le R$ to within $\epsilon$
with roughly $nR/\epsilon$ paired comparisons (ignoring log factors). This procedure is also robust to a variety of forms of noise. If one is
able to shift the distribution
of the items drawn, adaptive estimation
gives a substantial improvement over a non-adaptive strategy.  To directly
implement such a scheme, one would require the ability to generate items
arbitrarily in $\R^n$.  While there may be some cases
where this is possible (e.g., in market testing of items where the
features correspond to known quantities that can be manually manipulated, such
as the amount of various ingredients in a food or beverage), in many of the
settings considered by recommendation systems, the only items
which can be compared belong to a fixed set of points. While our theory
would still provide rough guidance as to how accurate of a localization is
possible, many open questions in this setting remain. For instance, the algorithm itself needs to be adapted, as done in
Section~\ref{sec:exper3}.
Of course, there are many other ways that the adaptive scheme
could be modified to account for this restriction.  For example, one could use
rejection sampling, so that although many candidate pairs would need to be
drawn, only a fraction would actually need to be presented to and labeled by
the user. We leave the exploration of such variations for future work.

 
\clearpage
\acks{This work was supported by grants AFOSR FA9550-14-1-0342, NSF CCF-1350616, and a gift from the Alfred P.\ Sloan Foundation.}

\vskip 0.2in
\frenchspacing
\bibliography{ref}

\clearpage\appendix
\section{Supporting lemmas}
\begin{lemma}\label{lem:cdflower}
Let $b > a$ and let $L = \min\{|a|,|b|\}$ and $U = \max\{|a|,|b|\}$. Then if $\Phi$ and $\phi$ respectively denote the standard normal cumulative distribution
function and probability distribution function, we have the bounds
\[
(b-a)\phi(U) \le \Phi(b)-\Phi(a) \le (b-a)\phi(L) \le (b-a)\phi(0).
\]
\end{lemma}
\begin{proof}
By the mean value theorem, we have for some $a<c<b$, $\Phi(b) - \Phi(a)= (b-a) \Phi'(c) = (b-a)\phi(c)$.
Since $\phi(|x|)$ is monotonic decreasing, it is lower bounded by $\phi(U)$ and upper bounded by $\phi(L)$ (and also $\phi(0)$).
\end{proof}

\begin{lemma} \label{lem:sphereint}
    Let $\x,\y\in\R^n$.  Then,
        \[
            \int_{\S^{n-1}} |\a^T(\x-\y)| \,\nu(\mathrm{d} \a)
    = \frac{2}{\sqrt\pi}
    \frac{\Gamma(\frac{n}{2})}{\Gamma(\frac{n+1}{2})}
    \norm{\x-\y}.
        \]
\end{lemma}
    \begin{proof}
        By spherical symmetry, we may
assume $\Delta = \x-\y = [\epsilon, 0,\dots, 0]$ for $\epsilon> 0$
without loss of generality.
Then $\norm{\x-\y} = \epsilon$ and $|\a^T(\x-\y)| = a(1) \epsilon = \epsilon\lvert\cos\theta\rvert$,
where $\cos^{-1}(a(1)) = \theta \in[0,\pi]$.
We will use the fact~\cite{gradshteyn2007}:
\[
    \int_0^{\frac\pi 2} \cos^{\mu-1}\theta \sin^{\omega-1}\theta
        \der\theta =
        \frac{1}{2}B \left( \frac{\mu}{2}, \frac{\omega}{2} \right)
        = \frac12 \frac{\Gamma(\mu/2)\Gamma(\omega/2)}
        {\Gamma((\mu+\omega)/2)}.
\]
Integrating $\lvert\cos\theta\rvert$ in the first spherical coordinate,
since the integrand is symmetric about $\frac{\pi}{2}$,
\[
    \int_0^\pi \lvert\cos\theta\rvert
    \sin^{n-2}\theta \der\theta
    = 2\int_0^{\pi/2} \cos\theta
    \sin^{n-2}\theta \der\theta
    = \frac{\Gamma(1)\Gamma(\frac{n-1}{2})}{\Gamma(1+\frac{n-1}{2})}
    = \frac{2}{n-1}.
\]
Then with the appropriate normalization, we have
(using $\Gamma(1/2) = \sqrt{\pi}$)
\begin{align*}
    \int_{S^{n-1}} |a^T(\x-\y)| \,\nu(\mathrm{d}a)
    &=
    \left(\int_0^\pi \sin^{n-2}\theta \der\theta\right)^{-1}
    \int_0^\pi \epsilon\lvert\cos\theta\rvert
    \sin^{n-2}\theta
    \der\theta
    \\&= \epsilon
    \left(\frac{\Gamma(\frac{1}{2})  \Gamma(\frac{n-1}{2})}
    {\Gamma(\frac{1}{2}+\frac{n-1}{2})}\right)^{-1}
    \kern-6pt\frac{2}{n-1}
    = \frac{2}{\sqrt\pi}
    \frac{\Gamma(\frac{n}{2})}{\Gamma(\frac{n+1}{2})}
    \norm{\x-\y}.
\end{align*}
~\vspace{-0.6in}\\\mbox{}
\end{proof}

\section{Integral calculations for Lemma~\ref{lem:gaussnoise1}}
\label{sec:noisedetails}

\paragraph{Bound \eqref{eq:Pqiqi}}
Recall that $r_i: = \a_i^T\x/\norm{\x} \in [-1,1]$, $q_i = r_i \norm{\x} - \tau_i$, and $\bar q_i = r_i \norm{\x} - \tau_i + z_i$.  Thus, if $f_r(r_i)$, $f_\tau(\tau_i)$, and $f_z(z_i)$ denote the probability density functions for $r_i$, $\tau_i$, and $z_i$, then since these random variables are independent we can write
\begin{align*}
\pr \left[ q_i < 0 ~~\text{and}~~ \bar q_i > 0 \right] & = \pr \left[ r_i \norm{\x} - \tau_i < 0 ~~\text{and}~~ r_i \norm{\x} - \tau_i  +z_i > 0 \right] \\
& = \int_{-1}^1 \int_{r_i\norm{\x}}^{\infty} \int_{-\infty}^{r_i \norm{\x} - \tau_i} f_r(r_i) f_\tau(\tau_i) f_z(z_i) \der z_i \der \tau_i \der r_i \\
& = \int_{-1}^1 \int_{r_i \norm{\x}}^{\infty} f_r(r_i) f_\tau(\tau_i) \pr \left[ z_i > \tau_i - r_i \norm{\x} \right] \der \tau_i \der r_i \\
& = \int_{-1}^1 \int_{r_i \norm{\x}}^{\infty} f_r(r_i) f_\tau(\tau_i) Q\left(\frac{\tau_i -r_i \norm{\x}}{\sigma_z}\right) \der \tau_i \der r_i,
\end{align*}
where $Q(x) = \frac{1}{\sqrt{2\pi}} \int_{x}^\infty \exp (-x^2/2) \der x$, i.e., the tail probability for the standard normal distribution. Via a similar argument we have
\begin{align*}
\pr \left[ q_i > 0 ~~\text{and}~~ \bar q_i < 0 \right] & = \pr \left[ r_i \norm{\x} - \tau_i > 0 ~~\text{and}~~ r_i \norm{\x} - \tau_i  +z_i < 0 \right] \\
& = \int_{-1}^1 \int_{-\infty}^{r_i \norm{\x}} \int_{r_i \norm{\x} - \tau_i}^{\infty} f_r(r_i) f_\tau(\tau_i) f_z(z_i) \der z_i \der \tau_i \der r_i \\
& = \int_{-1}^1 \int_{-\infty}^{r_i \norm{\x}} f_r(r_i) f_\tau(\tau_i) \pr \left[ z_i <  \tau_i - r_i \norm{\x}) \right] \der \tau_i \der r_i \\
& = \int_{-1}^1 \int_{-\infty}^{r_i \norm{\x}} f_r(r_i) f_\tau(\tau_i) Q\left(\frac{r_i \norm{\x} - \tau_i}{\sigma_z}\right) \der \tau_i \der r_i.
\end{align*}
Combining these we obtain
\begin{align*}
    \pr[q_i\bar q_i<0] & = \pr \left[ q_i < 0 ~~\text{and}~~ \bar q_i > 0 \right] + \pr \left[ q_i > 0 ~~\text{and}~~ \bar q_i < 0 \right] \\
    & = \int_{-1}^{1} \int_{-\infty}^\infty f_r(r_i) f_\tau(\tau_i) Q\left(\frac{|r_i \norm{\x} - \tau_i|}{\sigma_z}\right) \der \tau_i \der r_i \\
    & = 2 \int_0^1\int_{-\infty}^\infty f_r(r_i) f_\tau(\tau_i) Q\left(\frac{|r_i \norm{\x} - \tau_i|}{\sigma_z}\right) \der \tau_i \der r_i,
\end{align*}
following from the symmetry of $f_r(\cdot)$.
Using the bound $Q(x) \le \frac12 \exp(-x^2/2)$ \citep[see (13.48) of][]{JohnsKB_Continuous}, and recalling that $\tau_i \sim \cN(0,2R^2/n)$, we have that
\begin{equation*}
\pr[ q_i\bar q_i<0 ] \le \frac{1}{R} \sqrt{\frac{n}{\pi}} \int_{0}^1 \int_{-\infty}^\infty f_r(r_i) \exp\left(-\frac{(r_i \norm{\x}-\tau_i)^2}{2\sigma_z^2}-\frac{n \tau_i^2}{4R^2}\right) \der \tau_i \der r_i.
\end{equation*}

\subsection{Bounding $\kappa_n$}
First, we give an expression for $\kappa_n$ for all cases $n\ge 2$,
expanding upon that given in Theorem~\ref{thm:gaussnoise} and
Lemma~\ref{lem:gaussnoise1}.
We have
\[ \kappa_n(\sigma_z^2) :=
\begin{cases}
    \frac12 \sqrt{\frac{\sigma_z^2}{\sigma_z^2 + R^2}}
    & n = 2      \\
    \min \left\{ 
    \sqrt{\frac{\sigma_z^2}{
        \sigma_z^2+2R^2/3
    }} \, ,~
    \sqrt{\frac{\pi}{2}} \frac{\sigma_z}{\norm{\x}}
    \;\right\} & n =3     \\
    \sqrt{\frac{\sigma_z^2}{
        \sigma_z^2 + 2R^2/n + 4\norm{\x}^2/n
    }}
    & n \ge 4.
\end{cases}
\]
Below we derive this expression for the cases $n=2$, $n=3$, and $n\ge 4$.

\subsection{Case \texorpdfstring{$n=2$}{n=2}}
For the special case $n=2$, $d_i = \cos\theta_i$ where
$\theta_i\in[-\pi,\pi]$ is distributed uniformly.
In this case,~\eqref{eq:Pqiqi} can be re-written as
\begin{align*}
    \pr[ q_i\bar q_i<0 ] & \le \frac{1}{2R}\sqrt{\frac{2}{\pi}}
    \int_{-\pi/2}^{\pi/2} \int_{-\infty}^{\infty}
        {\frac{1}{2\pi} \exp\left(
            -\frac{( \norm{\x} \cos\theta_i-\tau_i)^2}{2\sigma_z^2}
            -\frac{\tau_i^2}{2R^2}
        \right)   }
        \der\tau_i \der \theta_i \\
        & = \frac{1}{\pi R}\sqrt{\frac{1}{2\pi}}
        \int_{0}^{\pi/2} \int_{-\infty}^{\infty}
        { \exp\left(
            -\frac{( \norm{\x} \cos\theta_i-\tau_i)^2}{2\sigma_z^2}
            -\frac{\tau_i^2}{2R^2}
        \right)   }
        \der\tau_i \der \theta_i.
\end{align*}
Expanding and setting $\alpha$, $\beta$, and $\gamma$ appropriately,
\[\bal
    \pr[ q_i\bar q_i<0 ] &\le
    \frac{1}{\pi R}\sqrt{\frac{1}{2\pi}}
    \int_{0}^{\pi/2}\int_{-\infty}^{\infty}
    \exp\left(-\frac{\norm{\x}^2 \cos^2\theta_i}{2\sigma_z^2}
            +\frac{2\norm{\x}\tau_i \cos\theta_i}{2\sigma_z^2}
            -\frac{\tau_i^2}{2\sigma_z^2}
            - \frac{\tau_i^2}{2R^2}
        \right)
        \der\tau_i \der \theta_i
    \\ &=
    \frac{1}{\pi R}\sqrt{\frac{1}{2\pi}}
    \int_{0}^{\pi/2}\int_{-\infty}^{\infty}
    \exp\left(
        -\gamma\cos^2\theta_i + \beta\tau_i \cos\theta_i - \alpha\tau_i^2
    \right)
    \der\tau_i\der \theta_i.
\eal\]
Completing the square for $\tau_i$,
\[\bal
    \pr[ q_i\bar q_i<0 ] &=
    \frac{1}{\pi R}\sqrt{\frac{1}{2\pi}}
    \int_{0}^{\pi/2}\int_{-\infty}^{\infty}
    \exp\left(-\alpha \left(\tau_i+ \frac{\beta\cos\theta_i}{2\alpha}\right)^2
        +\frac{(\beta \cos\theta_i)^2}{4\alpha} -\gamma\cos^2\theta_i
        \right)
        \der\tau_i \der\theta_i
    \\ &=
    \frac{1}{\pi R}\sqrt{\frac{1}{2\pi}}
    \int_{0}^{\pi/2} \sqrt{\frac{\pi}{\alpha}}
        \exp\left(-\left(\gamma-\frac{\beta^2}{4\alpha}\right) \cos^2\theta_i
        \right)
    \der\theta_i \\
       & =
    \frac{\pi}{2\pi R}\sqrt{\frac{1}{2\alpha}}
    \exp\left(- \frac12 \left(\gamma-\frac{\beta^2}{4\alpha}\right)
    \right)I_0\left(
        \frac12 \left(\gamma-\frac{\beta^2}{4\alpha}\right)
    \right),
\eal\]
where $I_0(\cdot)$ denotes the modified Bessel function of the first kind.
Since $\exp(-t)I_0(t) < 1$, by plugging back in for $\alpha$ we obtain
\[
    \pr[ q_i\bar q_i<0 ] \le
    \frac{1}{2R}\sqrt{\frac{1}{2\alpha}}
     =
     \frac{1}{2\sqrt{2}R}
    \sqrt{\frac{1}{
         \frac{1}{2\sigma_z^2}+\frac{1}{2R^2}
    }}
    =
    \half \sqrt{\frac{\sigma_z^2}{
        \sigma_z^2 + R^2
    }}.
\]
We also note that since since $\exp(-t)I_0(t) < 1/\sqrt{\pi t}$, we can obtain the bound $\pr[ q_i\bar q_i<0 ] \le \frac{1}{\sqrt{\pi}} \frac{\sigma_z}{\norm{\x}}$, but one can show that the previous bound will dominate this whenever $\norm{\x} \le R$.

\subsection{Case \texorpdfstring{$n=3$}{n=3}}
For the case $n=3$, $d_i\sim[-1,1]$ is itself distributed uniformly.
In this case we have
\[
    \pr [ q_i\bar q_i<0 ]
    \le \frac{1}{2 R}\sqrt{\frac{3}{\pi}}
    \int_{0}^{1} \int_{-\infty}^{\infty}
        { \exp\left(
            -\frac{(d_i \norm{\x}-\tau_i)^2}{2\sigma_z^2}
            -\frac{3 \tau_i^2}{4R^2}
        \right)   }
        \der\tau_i \der d_i.
\]
Expanding and setting $\alpha$, $\beta$, and $\gamma$
appropriately,
\[\bal
    \pr [ q_i\bar q_i<0 ] &\le
    \frac{1}{2R}\sqrt{\frac{3}{\pi}}
    \int_{0}^{1}\int_{-\infty}^{\infty}
    \exp\left(-\frac{d_i^2\norm{\x}^2}{2\sigma_z^2}
            +\frac{2d_i\norm{\x}\tau_i}{2\sigma_z^2}
            - \frac{\tau_i^2}{2\sigma_z^2}
            - \frac{3 \tau_i^2}{4R^2}
        \right)
        \der\tau_i \der d_i
    \\ &=
    \frac{1}{2R}\sqrt{\frac{3}{\pi}}
    \int_{0}^{1}\int_{-\infty}^{\infty}
    \exp\left(
        -\gamma d_i^2 + \beta d_i\tau_i - \alpha\tau_i^2
    \right)
    \der\tau_i\der d_i.
\eal\]
Completing the square for $\tau_i$,
\[\bal
    \pr [ q_i\bar q_i<0 ] &=
    \frac{1}{2R}\sqrt{\frac{3}{\pi}}
    \int_{0}^{1}\int_{-\infty}^{\infty}
    \exp\left(-\alpha \left(\tau_i+\frac{d_i \beta}{2\alpha}\right)^2
        +\frac{(d_i\beta)^2}{4\alpha} -\gamma d_i
        \right)
        \der\tau_i \der d_i
\\ &=
    \frac{1}{2R}\sqrt{\frac{3}{\pi}}
    \int_{0}^{1} \sqrt{\frac{\pi}{\alpha}}
        \exp\left(-d_i^2 \left(\gamma-\frac{\beta^2}{4\alpha}\right)
        \right)
    \der d_i
\\ &=
    \frac{1}{2R}\sqrt{\frac{3 }{\alpha}} \frac{\sqrt{\pi}}{2}
    \frac{
        \operatorname{erf}\left(\sqrt{\gamma-\beta^2/4\alpha}\right)
    }{\sqrt{\gamma-\beta^2/4\alpha}}.
\eal\]
Since $\operatorname{erf}(t)/t\le 2/\sqrt{\pi}$, by plugging back in for $\alpha$ we obtain
\[
    \pr [ q_i\bar q_i<0 ] \le
    \frac{1}{2R}\sqrt{\frac{3}{
        \frac{1}{2\sigma_z^2}+\frac{3}{4R^2}
    }}
    =
    \sqrt{\frac{\sigma_z^2}{
        \sigma_z^2 + 2R^2/3
    }}.
\]
Additionally, since
$\operatorname{erf}(t)\le 1$,
\begin{align*}
    \pr[ q_i\bar q_i<0 ] &\le
    \frac{\sqrt{3 \pi}}{4R}
    \left(\gamma\alpha-\beta^2/4 \right)^{-1/2} \\
    & = \frac{\sqrt{3 \pi}}{4 R}
    \left( \frac{\norm{\x}^2}{2\sigma_z^2}
        \left( \frac{1}{2\sigma_z^2} + \frac{3}{4R^2} \right)
        - \frac{\norm{\x}^2}{4 \sigma_z^4}
    \right)^{-1/2}
    \\ &= \frac{\sqrt{3 \pi}}{4R} \left( \frac{3 \norm{\x}^2}{8 \sigma_z^2 R^2} \right)^{-1/2} \\
    & = \sqrt{\frac{\pi}{2}} \frac{\sigma_z}{\norm{\x}},
\end{align*}
which can be tighter when $\sigma_z$ is small and $\norm{\x}$ is large.

\subsection{Case \texorpdfstring{$n \ge 4$}{n greater than 3}}
Combining~\eqref{eq:Pqiqi} with our upper bound \eqref{eq:ubfd} on
$f_d(d_i)$, we obtain
\[
    \pr[ q_i\bar q_i<0 ]
    \le \frac{n}{4 \sqrt{2} \pi R}
    \int_{0}^{1} \int_{-\infty}^{\infty}
        { \exp\left(
            -\frac{(d_i \norm{\x}-\tau_i)^2}{2\sigma_z^2}
            -\frac{n\tau_i^2}{4R^2} - \frac{n d_i^2}{8}
        \right)   }
        \der\tau_i \der d_i.
\]
Expanding and setting $\alpha$, $\beta$, and $\gamma$
appropriately,
\[\bal
    \pr [ q_i\bar q_i<0 ] &\le
     \frac{n}{4 \sqrt{2} \pi R}
    \int_{0}^{1}\!\int_{-\infty}^{\infty}
    \exp\left(-d_i^2
        \left(
            \frac{\norm{\x}^2}{2\sigma_z^2}
            + \frac{n}{8}
        \right)
            +\frac{2d_i\norm{\x}\tau_i}{2\sigma_z^2}
            -\frac{\tau_i^2}{2\sigma_z^2}
            -\frac{n\tau_i^2}{4R^2}
        \right)
        \der\tau_i \der d_i
    \\ &=
    \frac{n}{4 \sqrt{2} \pi R}
    \int_{0}^{1}\int_{-\infty}^{\infty}
    \exp\biggl(
        -\gamma d_i^2 + \beta d_i\tau_i - \alpha\tau_i^2
    \biggr)
    \der\tau_i\der d_i.
\eal\]
Completing the square for $\tau_i$,
\begin{align*}
    \pr[ q_i\bar q_i<0 ] &=
    \frac{n}{4\sqrt{2} \pi R}
    \int_{0}^{1}\int_{-\infty}^{\infty}
    \exp\left(-\alpha \left(\tau_i-\frac{d_i\beta}{2\alpha}\right)^2
        +\frac{(d_i\beta)^2}{4\alpha} -\gamma d_i^2
        \right)
        \der\tau_i \der d_i
\\ &=
    \frac{n}{4\sqrt{2} \pi R}
    \int_{0}^{1} \sqrt{\frac{\pi}{\alpha}}
        \exp\left(-d_i^2 \left(\gamma-\frac{\beta^2}{4\alpha}\right)
        \right)
    \der d_i
\\ &=
    \frac{n}{4\sqrt{2\pi\alpha}R}
    \frac{\sqrt\pi}{2}
    \frac{
        \operatorname{erf}\left(\sqrt{\gamma-\beta^2/4\alpha}\right)
    }{\sqrt{\gamma-\beta^2/4\alpha}}.
\end{align*}
Since $\operatorname{erf}(t) \le 1$, we have
\begin{align*}
    \pr[ q_i\bar q_i<0 ] &
\le
    \frac{n}{4\sqrt{2}R}
    \left(\gamma\alpha-\beta^2/4\right)^{-1/2}
\\ &=
    \frac{n}{4\sqrt{2}R}
    \left(
        \left(\frac{\norm{\x}^2}{2\sigma_z^2} + \frac{n}{8}\right)
        \left(
            \frac{1}{2\sigma_z^2} + \frac{n}{4R^2}
        \right)
        -
        \frac{\norm{\x}^2}{4 \sigma_z^4}
    \right)^{-1/2}
    \\ &=  \half
    \left( \frac{32 R^2}{n^2}
    \left(\frac{\norm{\x}^2 n}{8 \sigma_z^2 R^2} + \frac{n}{16 \sigma_z^2} + \frac{n^2}{32 R^2}
    \right) \right)^{-1/2}
    \\ &= \half
    \sqrt{\frac{\sigma_z^2}{\sigma_z^2 + 2R^2/n + 4\norm{\x}^2/n   }}.
\end{align*}
We also note that since $\operatorname{erf}(t)/t\le 2/\sqrt{\pi}$, it is also possible to obtain the bound
\[
\pr[ q_i\bar q_i<0 ] \le \half\sqrt{ \frac{n}{2\pi} } \sqrt{\frac{\sigma_z^2}{\sigma_z^2 + 2 R^2/n}}.
\]
However, this bound can only be tighter when $\norm{\x}$ is small and when $\frac{n}{2\pi} < 1$ (i.e., for $n \le 6$).  Given this narrow range of applicability, we omit this from the formal statement of the result.

 
\end{document}